\documentclass{article} 

\usepackage[preprint]{myarxiv}

\usepackage[utf8]{inputenc} 
\usepackage[T1]{fontenc}    
\usepackage{hyperref}       
\usepackage{url}            
\usepackage{booktabs}       
\usepackage{amsfonts}       
\usepackage{nicefrac}       
\usepackage{microtype}      
\usepackage{xcolor}         
\usepackage{graphicx} 
\usepackage{amsmath}
\usepackage{amssymb}
\usepackage{amsthm}
\usepackage{tikz}
\usetikzlibrary{arrows.meta}
\usetikzlibrary{intersections}
\usetikzlibrary{calc}
\usepackage{natbib}
\usepackage{xfrac}
\usepackage{subcaption}
\usepackage{mathtools}
\usepackage{hyperref}
\usepackage{cleveref}
\usepackage{bbm}
\usepackage{wrapfig}

\usepackage{todonotes}
\usepackage{enumitem}

\usepackage{thmtools}

\makeatletter
\newcommand{\listofappendices}{%
  \section*{Appendix Contents}%
  \@starttoc{apx}%
}
\makeatother

\setcitestyle{authoryear,round,comma}

\newtheorem{theorem}{Theorem}

\newtheorem{prop}[theorem]{Proposition}
\newtheorem{lemma}[theorem]{Lemma}
\newtheorem{cor}[theorem]{Corollary}

\theoremstyle{definition}
\newtheorem{definition}[theorem]{Definition}

\newtheorem{remark}[theorem]{Remark}

\usepackage{xcolor}

\newcommand{\aff}{\operatorname{aff}}
\newcommand{\R}{\mathbb{R}}
\DeclareMathOperator{\rank}{rank}
\newcommand{\architecture}{\mathcal{A}}
\newcommand{\paramspace}[1]{\Theta_{#1}}
\newcommand{\gparamspace}[1]{\Tilde{\Theta}_{#1}}
\newcommand{\weights}[1]{W^{(#1)}}

\newcommand{\bias}[1]{b^{(#1)}}
    \newcommand{\complex}{\mathcal{C}}
\newcommand{\activation}[1]{a^{(#1)}}
\newcommand{\diag}{\operatorname{diag}}
\newcommand{\ind}{\mathbbm{1}}
\newcommand{\preactivation}[1]{z^{(#1)}}
\newcommand{\bhyperplane}[1]{B_{#1}}

\newcommand{\hyperplanes}[2]{\mathcal{H}_{#1}(#2)}
\newcommand{\breakpoints}[1]{B({#1})}
\newcommand{\bcomplex}[1]{\mathcal{B}_{#1}}
\newcommand{\funspace}[1]{\mathcal{F}_{#1}}
\newcommand{\fiber}[1]{\mathcal{S}({#1})}
\newcommand{\gfiber}[1]{\Tilde{\mathcal{S}}({#1})}

\newcommand{\str}{\operatorname{star}}

\newcommand{\inner}[1]{\langle #1 \rangle}

\newcommand{\relint}{\operatorname{relint}}

\newcommand{\lin}{\operatorname{lin}}

\newcommand{\orthant}[1]{\R_{\geq 0}^{#1}}

\newcommand{\modulo}[1]{\big/#1}
\newcommand{\nncr}[1]{\mathrm{rank}_{\mathrm{col},+}(#1)}
\DeclareMathOperator{\GL}{GL}
\newcommand{\cone}{\operatorname{cone}}
\newcommand{\col}{\operatorname{col}}

\title{The Symmetries of Three-Layer ReLU Networks}

\author{%
   Johanna Marie Gegenfurtner \\
   Technical University of Denmark \\
   \texttt{johge@dtu.dk} \\
    \And
   Moritz Grillo \\
  Max Planck Institute for Mathematics in the Sciences\\
    Leipzig\\
  \texttt{moritz.grillo@mis.mpg.de} \\
  \And
   Guido Mont\'ufar \\ 
   UCLA and MPI MiS\\
   \texttt{montufar@math.ucla.edu}
}

\begin{document}

\maketitle

\begin{abstract}
We develop a framework for analyzing parameter symmetries in deep ReLU networks and obtain a complete characterization of the generic parameter fibers for three-layer bottleneck architectures. Our approach provides explicit semi-algebraic descriptions of these fibers and yields a polynomial time algorithm for deciding functional equivalence of two parameters. The symmetries include discrete and continuous transformations arising from layer composition, and depend on whether deeper layers hide or preserve geometric structure from preceding layers. Finally, we show that some of these symmetries induce local conservation laws along gradient flow, while others do not. 
\end{abstract}

\section{Introduction}

A fundamental property of neural network parameterizations is parameter redundancy: distinct parameter configurations can realize the same function. 
These redundancies can be formalized via the notion of \emph{fibers}, defined as the set of all parameters realizing a given function. 
Symmetries, i.e., parameter transformations that leave the function unchanged, provide a canonical mechanism for generating such fibers. 
For instance, in ReLU networks, neuron weights can be permuted within layers, and incoming and outgoing weights rescaled without affecting the function 
\citep{Kurkova1994, rolnick2020reverse}. 
Beyond such global symmetries, fibers also arise from \emph{subnetwork} structure: when the realized function depends only on a subset of parameters, the rest can vary freely without changing the function. 
Such symmetries significantly influence the optimization landscape and can induce conserved quantities along training, constraining how gradient methods explore parameter space and which solutions are ultimately selected \citep{głuch2021noetherthingschangestay,
kunin2021neural,
zhao2023symmetries,
zhao2025symmetryneuralnetworkparameter,
nurisso2026topology}. 

In ReLU networks, parameter symmetries are relatively well understood in the single-hidden-layer case. 
In particular, \citet{PetzkaTS20} show that, for generic parameters, the fiber is trivial up to permutation and scaling, and provide a description in the remaining degenerate cases. 
In deep architectures, the situation is more nuanced. While global symmetries are still generally limited to permutation and scaling, additional local symmetries can arise at specific non-degenerate parameter configurations \citep{rolnick2020reverse,
grigsby2023hiddensymmetriesrelunetworks}. 
At the same time, for many ReLU architectures there exist functions whose parameters are identifiable (i.e., uniquely determined) among suitable classes of parameters up to these global symmetries \citep{phuong2020functional,grigsby2023hiddensymmetriesrelunetworks}.
However, identifiability in ReLU networks is inherently non-uniform and the degree of redundancy can vary significantly across parameter space \citep{ELISENDAGRIGSBY2025110636}. 
In particular, the existence of identifiable parameter configurations does not imply identifiability throughout generic parameters. 

As a result, beyond well-understood global symmetries, the structure of fibers in ReLU networks with more than one hidden layer remains largely uncharacterized. 
In particular, a systematic description of all parameters realizing a given function, and of the interaction between different sources of redundancy, is still missing. 
Addressing this problem for two-hidden-layer networks is a key step toward understanding identifiability, minimality, and parameter optimization in deep ReLU networks.

\subsection{Our Contributions}

We develop a systematic framework for describing symmetries and fibers in ReLU neural networks beyond the standard scaling and permutation symmetries. 
Our contributions are as follows: 
\begin{enumerate}[leftmargin=*, itemsep=.3pt, topsep=.3pt]
    \item \textbf{Layerwise symmetries:} 
    We characterize symmetries acting within individual layers by analyzing two-layer networks on the nonnegative orthant. 
    In this setting, we obtain explicit descriptions of fibers via polynomial equations and inequalities and compute the dimension of generic fibers. 
    These results provide a fundamental building block for the analysis of deeper architectures. 
    
    \item \textbf{Symmetries induced by layer composition:} 
    We identify new symmetries that emerge from composing ReLU layers by analyzing three-layer networks. We show that layer composition introduces redundancies that cannot be reduced to symmetries acting within individual layers. 

    \item \textbf{Explicit fiber description for three-layer architectures:} 
    We show the above symmetries are complete for three-layer bottleneck networks with generic parameters. 
    Remarkably, these fibers depend only on weight signs 
    and imply a polynomial time algorithm for deciding functional equivalence of parameters. 

    \item \textbf{Non-localizability of fibers:} 
    We show that symmetries in deep networks are not always localizable. We construct networks where each pair of consecutive layers is identifiable, but their composition is not, showing that parameter redundancy can arise as a global property of network composition. 

    \item \textbf{Connections to training dynamics:} 
    We show that local symmetries from linearly acting neurons induce conserved quantities along gradient flow, but the symmetries from \Cref{sec:fibers_from_concatenation} do not.  
\end{enumerate}

\subsection{Related Work}

The study of parameter equivalence in neural networks dates back at least to the work of \citet{SUSSMANN1992589,Kurkova1994,Fefferman1994ReconstructingAN} for tanh and asymptotically constant activations. 
For networks with polynomial activations, recent works show, depending on the degree, that generic parameters are identifiable up to permutation and scaling \citep{shahverdi2025learningrazorsedgesingularity, 
usevich2025identifiabilitydeeppolynomialneural, 
finkel2025activationdegreethresholdsexpressiveness}. 
More generally, \citet{vlacic2020affinesymmetriesneuralnetwork} develop a framework in which parameter fibers are generated by affine symmetries of the activation function, leading to equivalence classes beyond permutation and scaling. 
For ReLU activations, however, they indicate that additional, non-affine symmetries would need to be taken into account to characterize the fibers. 

For ReLU networks, parameter symmetries are particularly intricate. 
In the shallow setting, \citet{PetzkaTS20,
ramakrishnan2026completesymmetryclassificationshallow} give a complete
classification, showing that the generic fiber is trivial up to permutation and scaling, and characterizing degenerate cases where subsets of neurons aggregate into a single linear function. In a related direction, \citet{DEREICH2022121}
characterize minimal representations of scalar-valued functions realized by shallow ReLU
networks and determine the dimension and number of connected components of
the corresponding function space. Building on this, \citet{Dereich_2024} prove the existence of optimal
approximations of continuous target functions within this function space. Their tools are closely related to the
hyperplane representation used in our work, which generalizes this viewpoint
to intermediate layers with nonnegative inputs and to non-scalar outputs. 
 
Our approach complements the reverse-engineering framework of \citet{rolnick2020reverse}, which leverages the combinatorial geometry of activation boundaries to recover parameters. Relatedly, \citet{grigsby2023hiddensymmetriesrelunetworks} show that architectures with widths at least the input dimension admit open sets of parameters that are identifiable within a suitable parameter class, while \citet{phuong2020functional} prove analogous results for architectures with non-increasing widths. However, the structure of the fibers outside these identifiable sets remains largely unknown.

We further highlight the work of \citet{Stock2022}, which introduces a path-based locally linear parametrization of the realization map and derives conditions for local identifiability. 
Complementary local analyses by \citet{bona-pellissier2022local,Bona-Pellissier2023} provide conditions under which deep ReLU networks are identifiable. 
From a different perspective, \citet{NEURIPS2019_04115ec3} study the degeneracy of the ReLU parametrization via inverse stability,  
showing the realization map can be highly non-injective. 
To our knowledge, a general description of parameter fibers in ReLU networks with more than one hidden layer, particularly accounting for compositional effects, remains unknown. 

Our discussion of training dynamics is motivated by the well-known fact that parameter symmetries can induce conservation laws during training. 
For example, matrix product symmetries preserve imbalance quantities \citep{tarmoun21understanding,kunin2021neural}, a line of work further developed by \cite{le2022training,marcotte2023abide}. 
These conserved quantities are of interest because they constrain the set of reachable parameters and provide insight into implicit bias and generalization. 
Recent work further characterizes initialization-dependent invariant sets of gradient flow and shows that their singularities correspond to submodels \citep{nurisso2026topology}. 
This raises the question of which additional conserved quantities may arise from other symmetry mechanisms.

\section{Preliminaries}

For $n \in \mathbb{N}$, we write $[n] \coloneqq \{1,\dots,n\}$. 
For $x \in \R^d$, we denote by $[x]_+$ the entrywise application of the ReLU activation function. 
For a set of indices $S \subseteq [n]$, let $D_S \in \R^{n\times n}$ denote the diagonal matrix with $(D_S)_{ii}=1$ if $i\in S$ and $(D_S)_{ii}=0$ otherwise. 
We write $\|\cdot\|$ for the Euclidean norm.

\paragraph{ReLU networks.}
A \emph{ReLU layer} with $n_{\ell-1}$ inputs and $n_\ell$ outputs is the map 
$
\phi^{(\ell)} 
\colon \R^{n_{\ell-1}} \to \R^{n_\ell};\,
x \mapsto [\weights{\ell}x + \bias{\ell}]_+
$,
parametrized by weight matrix $\weights{\ell} \in \R^{n_\ell \times n_{\ell-1}}$ and bias vector $\bias{\ell}\in\R^{n_\ell}$. 
A feedforward \emph{ReLU network} with architecture
$\architecture=(n_0,\ldots,n_{L+1})$ 
is the composition
$f_\theta(x)
=
T^{(L+1)}
\circ
\phi^{(L)} 
\circ \cdots \circ
\phi^{(1)}
(x)$, 
where the final layer
$
T^{(L+1)}
(y)=\weights{L+1}y+\bias{L+1}
$
is affine. 
We write $d=n_0$ and $m=n_{L+1}$ for the input and output dimensions. 
The parameter 
$
\theta=(\weights{1},\bias{1},\dots,\weights{L+1},\bias{L+1})
$
lies in the \emph{parameter space} 
$
\paramspace{\architecture}
\cong
\bigoplus_{\ell=1}^{L+1}
\bigl(\R^{n_\ell\times n_{\ell-1}}\times\R^{n_\ell}\bigr)
$. 
The associated \emph{realization map} is
$
\mu_\architecture \colon \paramspace{\architecture} \to \funspace{\architecture};\,
\theta \mapsto f_\theta
$. 
For $\ell\in[L]$, define the \emph{preactivation} and \emph{activation} at layer $\ell$ by
$
\preactivation{\ell,\theta}(x)
\coloneqq
\weights{\ell}\activation{\ell-1,\theta}(x)+\bias{\ell}$ and 
$\activation{\ell,\theta}(x)
\coloneqq
[\preactivation{\ell,\theta}(x)]_+$, 
with $\activation{0,\theta}(x)=x$. 
For fixed $\theta$, the maps $f_\theta$, $\preactivation{\ell,\theta}$, and $\activation{\ell,\theta}$ are continuous and piecewise linear. For a parameter 
$
\theta \in \paramspace{\architecture},
$
we also use $\weights{\theta,\ell}$ and $\bias{\theta,\ell}$ to denote the weight matrix and bias vector of the $\ell$-th layer, but omit $\theta$ in the superscript when there is no risk of ambiguity. 

\paragraph{Fibers and symmetries.} 
The \emph{fiber} of a parameter $\theta\in\paramspace{\architecture}$ is 
$
\fiber{\theta}
\coloneqq
\{\eta\in\paramspace{\architecture}\mid f_\eta=f_\theta\}
$. 
The parameter space of ReLU networks carries a standard equivalence relation generated by: 
\begin{enumerate}[leftmargin=*]
    \item permutation of neurons within any hidden layer, obtained by permuting the rows of $[\weights{\ell},\bias{\ell}]$, and inversely permuting the columns of $\weights{\ell+1}$; 
    \item positive rescaling of a hidden neuron, obtained by multiplying the corresponding row of $[\weights{\ell},\bias{\ell}]$ by $\lambda>0$, and multiplying the corresponding column of $\weights{\ell+1}$ by $\lambda^{-1}$. 
\end{enumerate}
We write $\theta\sim\eta$ if $\eta$ is obtained from $\theta$ by a sequence of such transformations. 
A parameter $\theta$ is called \emph{identifiable} if
$f_\eta=f_\theta$ implies $\eta\sim\theta$. 

\paragraph{Polyhedral complexes and piecewise linear functions.}
A \emph{polyhedral complex} $\complex$ in $\R^d$ is a finite collection of polyhedra containing the empty set such that if $P\in\complex$, then every face of $P$ belongs to $\complex$, and if $P,Q\in\complex$ with $P\cap Q\neq\emptyset$, then $P\cap Q$ is a face of both. 
We write $\complex^k$ for the $k$-faces of $\complex$, and call $\complex^{d-1}$ the \emph{facets} and $\complex^{d}$ the \emph{regions}. 
For a hyperplane arrangement $\mathcal{H}$, we write $\complex_\mathcal{H}$ for the induced polyhedral complex. 
A function $f\colon\R^d\to\R^m$ is \emph{continuous  piecewise linear} (CPWL) if there exists a complete polyhedral complex $\complex$ such that $f|_P$ is affine for each $P\in\complex$; in this case, $f$ and $\complex$ are \emph{compatible}. 
A point $x\in\R^d$ is a \emph{breakpoint} of 
$f$ if no neighborhood of $x$ exists on which $f$ is affine. 
We denote the set of breakpoints by $\breakpoints{f}$. 
For further details, see \Cref{app:polyheadral_geometry}. 

\paragraph{Weighted complexes.} 
Let $f\colon\R^d\to\R^m$ be CPWL and compatible with a polyhedral complex $\complex$. 
For a facet $\sigma\in\complex^{d-1}$, let $P,Q\in\complex^d$ be the adjacent maximal cells with $P\cap Q=\sigma$, and suppose 
$
f(x)=A_Px+b_P  \text{ for } x\in P 
$ and $
f(x)=A_Qx+b_Q  \text{ for } x\in Q
$. 
If $e_{P/\sigma}$ denotes the unit normal to $\sigma$ pointing from $Q$ into $P$, define the \emph{weight} of $f$ along $\sigma$ by 
$
c_f(\sigma)\coloneqq (A_P-A_Q)e_{P/\sigma}\in\R^m
$. 
The map
$
c_f\colon \complex^{d-1}\to\R^m
$
records the \emph{gradient jump} of $f$ across facets. 
The following standard fact from polyhedral and tropical geometry will be used repeatedly. More details are provided in Appendix~\ref{app:polyheadral_geometry}. 

\begin{theorem}[\citealp{maclagan2015introduction}]
\label{lem:structure_theorem_tropical_geometry}
Let $f \colon \R^d \to \R^m$ be CPWL and compatible with a polyhedral complex $\complex$. 
Then the weight function $c_f \colon \complex^{d-1}\to\R^m$ uniquely determines $f$ up to 
a global affine linear function.
\end{theorem}

\section{Polyhedral Geometry of ReLU Networks}
\label{sec:polyhedral_geometry_relu} 

In this section, we recall polyhedral structures naturally associated with ReLU networks that will be needed for the fiber analysis. 
Additional background and proofs are deferred to \Cref{app:relu_geometry}. 

\paragraph{Canonical polyhedral complex and bent hyperplanes.} 
For a ReLU network $f_\theta$ with architecture
$\architecture=(n_0,n_1,\dots,n_L,n_{L+1})
$, 
the \emph{canonical polyhedral complex} $\complex_\theta$ is constructed iteratively: 
starting from $\complex_{\theta,0}=\R^d$, one subdivides each cell of $\complex_{\theta,\ell-1}$ by the pulled back hyperplanes induced by the preactivations of layer $\ell$. 
The resulting complex $\complex_\theta=\complex_{\theta,L}$
is indexed by global activation patterns, and $f_\theta$ restricts to an affine-linear map on every maximal cell. 
For a neuron $(j,\ell)$, the preactivation $\preactivation{\ell,\theta}_j$ is affine linear on each cell $R\in\complex_{\theta,\ell-1}$. On cells where it is non-constant, its zero set is a codimension-one polyhedron and the union of these polyhedra is called the neuron's \emph{bent hyperplane}, denoted $B_{\ell,j}(\theta)$. 
We denote by
$
\hyperplanes{\ell}{\theta}
=
\{ \bhyperplane{\ell,j}(\theta) \mid j\in[n_\ell]\}
$ 
the set of bent hyperplanes of neurons in layer $\ell$. 
We note that not every facet of the canonical polyhedral complex is necessarily associated to a change in the linear behavior of the realized function. 
The \emph{breakpoint complex} is the subcomplex 
$\bcomplex{\theta}
\coloneqq
\{P\in \complex_\theta \mid P \subseteq \breakpoints{f_\theta}\}$, whose support is the breakpoint set: 
$|\bcomplex{\theta}|=\breakpoints{f_\theta}$. 
The canonical complex $\complex_\theta$ carries a weight function $c_\theta$ induced by $f_\theta$,
as in \Cref{lem:structure_theorem_tropical_geometry}. 
A facet of $\complex_\theta$ is visible in the
breakpoint complex precisely when the corresponding weight is nonzero: 
$\bcomplex{\theta}^{d-1}
=
\{\sigma\in \complex_\theta^{d-1} \mid c_\theta(\sigma)\neq 0\}$. 

\paragraph{Generic parameters.} 
Many of our results concern \emph{generic} parameters. At an informal level, genericity combines a
geometric and an algebraic requirement. Geometrically, the bent hyperplanes induced by the network
intersect in the expected codimensions, so that the canonical polyhedral complex has the expected
combinatorics. Algebraically, all 
masked products of weight matrices $W^{(\ell)}D_{S_{\ell-1}} W^{(\ell-1)}\cdots D_{S_k}W^{(k)}$ have maximal possible rank,
preventing degeneracies in the resulting linear parts, such as accidental alignments or
cancellations. The set of generic parameters is open and dense in $\paramspace{\architecture}$; we
denote it by $\gparamspace{\architecture} \subseteq \paramspace{\architecture}$. We refer to Appendix~\ref{app:generic-parameters} for a precise formulation.

\paragraph{Visibility of first-layer hyperplanes in three-layer networks.}
For three-layer networks, a first-layer hyperplane may or may not remain visible in the realized function,
depending on the activity of the second layer. 
The key mechanism is whether the hyperplane is \emph{anchored}, meaning that it is intersected by, or lies in
the active region of, a second-layer bent hyperplane.
The following result 
is illustrated in \Cref{fig:intersected_hyperplanes}
and proved in \Cref{sec:last_layer}. It will be central in the bottleneck analysis.

\begin{restatable}{theorem}{thmintersectedhyperplanesstay} 
\label{thm:intersected_hyperplanes_stay} 
Let $\architecture = (d,n_1,n_2,n_3)$ with $d,n_1>1$, and let $\theta,\eta\in\paramspace{\architecture}$
be generic parameters with $f_\theta=f_\eta$. 
Let $\hyperplanes{1}{\theta}$ denote the set of first-layer hyperplanes of $\theta$.
Then:
\begin{enumerate}[leftmargin=*]
    \item If $H\in \hyperplanes{1}{\theta}$ and there exists a neuron $j\in[n_2]$ together with
    a point $x\in H$ such that $\preactivation{2,\theta}_j(x)>0$, then $H\in \hyperplanes{1}{\eta}$.
    \item In particular, if this holds for every $H\in \hyperplanes{1}{\theta}$, then
    $
    \hyperplanes{1}{\theta}=\hyperplanes{1}{\eta}
$ and $
    \complex_\theta=\complex_\eta.
    $
\end{enumerate}
\end{restatable} 

\begin{figure}
\centering 
\begin{subfigure}{0.6\textwidth}
\begin{tikzpicture}[
    scale = .6, 
    line join=round,
    >={Latex[length=2mm,width=2mm]},
    cyanline/.style={draw=cyan!75!blue, line width=1.5pt},
    orangeline/.style={draw=orange!95!black, line width=1.5pt},
    yellowline/.style={draw=yellow!70!green, line width=3pt},
]

\begin{scope}[xshift=0cm]
    \draw[cyanline] (0,1.5) -- (3.5,3.5) -- (3.5,6);
    \draw[cyanline,->] (1.75,2.5) -- (1.75+.5*0.5714 ,2.5-.5);

    \draw[yellowline] (3.5,3.5) -- (6,1.41366);
    \draw[orangeline] (0.5,6) -- (3.5,3.5) -- (6,1.41366);

    \draw[yellowline] (1,0) -- (6,2.857);
    \draw[orangeline] (1,0) -- (6,2.857);
  
    \draw[black, line width=1pt] (0,0) rectangle (6,6);
\end{scope}

\begin{scope}[xshift=7cm]
\draw[cyanline] (0,1.5) -- (3.5,3.5) -- (3.5,6);
\draw[cyanline,->] (1.75,2.5) -- (1.75-.5*0.5714 ,2.5+.5);

\draw[yellowline] (0.5,6) -- (3.5,3.5);
\draw[orangeline] (0.5,6) -- (3.5,3.5) -- (6,1.41366);

\draw[orangeline] (1,0) -- (6,2.857);

\draw[black, line width=1pt] (0,0) rectangle (6,6);
\end{scope}

\end{tikzpicture} 
\end{subfigure}
\hspace{0.2cm}
\begin{subfigure}{0.3\textwidth}
\begin{tikzpicture}[
    scale = .6, 
    line join=round,
    >={Latex[length=1.5mm,width=1.5mm]},
    axis/.style={draw=black,->, line width=.8pt},
    orangeline/.style={draw=orange!95!black, line width=1pt},
    purpleline/.style={draw=violet!70, line width=1pt},
    redline/.style={draw=red!90!black, line width=1pt},
    blueline/.style={draw=blue!80, line width=1pt},
]

\def\arrowlen{15} 
\newcommand{\perparrow}[3][]{%
	\draw[#1,->]
	let
	\p1 = (#2),
	\p2 = (#3),
	\n1 = {\x2-\x1},
	\n2 = {\y2-\y1},
	\n3 = {veclen(\n1,\n2)}
	in
	($(#2)!0.5!(#3)$)
	-- ++({-\n2/\n3*\arrowlen},{\n1/\n3*\arrowlen});
}
    
\draw[axis] (-3,0) -- (3.75,0) node[right] {$x_1$};
\draw[axis] (0,-2.75) -- (0,3.5) node[above] {$x_2$};

\coordinate (A) at (-3,0.2);
\coordinate (B) at (1.3,-2.75);
\draw[violet!70, dashed, line width=1.2pt] (A) -- (B);
\perparrow[violet!70, dashed]{B}{A};
    
\coordinate (A) at (-2.0,1.9);
\coordinate (B) at (0.8,-2.75);
\draw[purpleline] (A) -- (B);    
\perparrow[purpleline]{A}{B}; 

\coordinate (Ared) at (0.0,0.7);
\coordinate (Bred) at (3.2,3.5);
\draw[redline, name path=redline] (Ared) -- (Bred); 

\coordinate (Ablue1) at (0.0,1.7);
\coordinate (Bblue1) at (3.7,0.6);
\draw[blueline, name path=blueline1] (Ablue1) -- (Bblue1);    

\coordinate (Ablue2) at (1.7,3.5);
\coordinate (Bblue2) at (3.05,0.0);
\draw[blueline, name path=blueline2] (Ablue2) -- (Bblue2);    

\path[name intersections={of=redline and blueline1, by= I1}]; 
\path[name intersections={of=redline and blueline2, by= I2}]; 
\path[name intersections={of=blueline1 and blueline2, by= I3}]; 
\coordinate (C) at ($(I1)!1/3!(I2)!1/3!(I3)$); 
\coordinate (I1') at ($(I1)!0.075!(C)$);
\coordinate (I2') at ($(I2)!0.075!(C)$);
\coordinate (I3') at ($(I3)!0.075!(C)$);

\fill[yellow, opacity=0.4]
(I1') -- (I2') -- (I3') -- cycle;

\perparrow[redline]{Bred}{Ared}; 
\perparrow[blueline]{Bblue1}{Ablue1}; 
\perparrow[blueline]{Ablue2}{Bblue2};

\perparrow[blueline, opacity=.2]{Ared}{Bred}; 
\coordinate (C) at (-2.0,1.5); 
\draw[purpleline, {Latex[length=1.3mm]}-{Latex[length=1.3mm]}, opacity=.3]
  ($(C)+(140:0.6)$) arc[start angle=160, end angle=80, radius=0.6];

\end{tikzpicture}
\end{subfigure}

\caption{\textbf{Left:} Illustration of \Cref{thm:intersected_hyperplanes_stay}. 
    The orange hyperplanes correspond to first-layer neurons, and the cyan bent hyperplane to a second-layer neuron. 
    The portions of the first-layer hyperplanes that intersect the active region of the second-layer neuron are highlighted in green. \textbf{Right:}
 Illustration of sign-flipping mechanism. 
    The region $R$ is shown in yellow. Hyperplanes of nonlinear neurons that are active on $R$ are shown in red, and those that are inactive in blue. A dead neuron is shown as a dashed,  and a linear neuron as a solid purple plane. 
    If the orientation of nonlinear neurons is reversed, this must be compensated by a corresponding change in the linear neurons, as described in (\ref{eq:generic_compensation}). 
    }
    \label{fig:intersected_hyperplanes}
    \label{fig:sign_flip}
\end{figure}
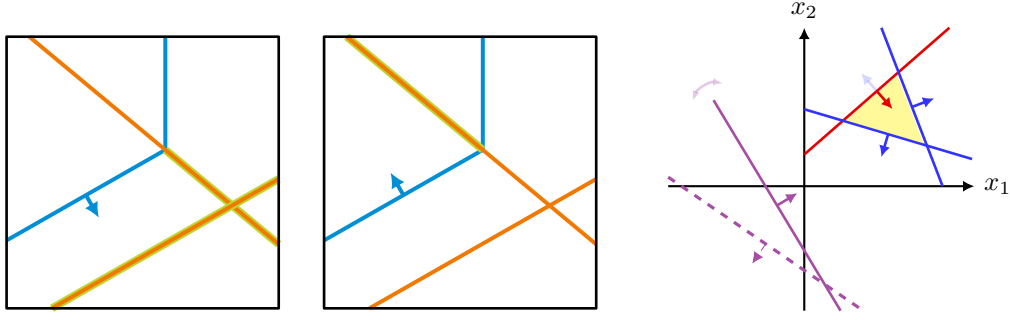

This theorem shows that first-layer hyperplanes that are functionally visible through the second layer
are rigid along the fiber. The only possible source of nontrivial fiber arising from the first layer in the three-layer
case is therefore the presence of \emph{hidden hyperplanes}, which we analyze in
Section~\ref{sec:fibers_from_concatenation}.

\section{Layerwise Fibers}
\label{sec:layerwise_fibers}

In this section, we provide a complete description of the fiber for a one-hidden-layer ReLU network $f_\theta \colon \R_{\ge 0}^d \to \R^m$. 
This is essential for understanding deep networks, as the output of any ReLU layer (except the final one) serves as a non-negative input to the subsequent layer. 
We first establish a \emph{hyperplane  representation} of $f_\theta$. 
We then describe the fiber as a union of semi-algebraic sets formed by intersections of algebraic varieties and polyhedral cones. 
Finally, we compute the dimension of the generic fiber components. 
All proofs belonging to this section can be found in \Cref{app:layerwise_fibers}. 

\subsection{Hyperplane Representation}
Let $X \subseteq \R^d$ be a full-dimensional polyhedron; in this paper we will only consider the cases
$X=\orthant{d}$ or $X=\R^d$.
A one-hidden-layer ReLU network $f_\theta \colon X \to \R^m$ is a CPWL function whose breakpoints are supported on a hyperplane arrangement
$
\mathcal{H}=\{H_1,\dots,H_k\}.
$
By \Cref{prop:tropical_weight_1layer} in the appendix, the weight function $c_\theta$ is constant along each breakpoint hyperplane: if
$\sigma,\sigma' \in \complex_\theta^{d-1}$ are facets contained in the same hyperplane $H_i$, then
$
c_\theta(\sigma)=c_\theta(\sigma').
$
Thus each breakpoint hyperplane $H_i$ carries a well-defined vector weight
$
c(i)\coloneqq c_\theta(\sigma)
$ for any facet $\sigma\subseteq H_i.
$

By \Cref{lem:structure_theorem_tropical_geometry}, the 
weighted hyperplane arrangement $(\mathcal{H},c)$ determines $f_\theta$ up to the addition of a global affine-linear function.
This remaining ambiguity is removed by fixing the affine map on a full-dimensional region $R$ of the arrangement.
This motivates the following definition. 

\begin{definition}
\label{def:hyperplane-representation}
A tuple $(\mathcal{H},A_R,b_R,c)$ is called a \emph{hyperplane representation} of
$f_\theta \colon X \to \R^m$ if:
\begin{enumerate}[leftmargin=*]
    \item $\mathcal{H}=\{H_1,\dots,H_k\}$ is the arrangement covering the breakpoints of $f_\theta$ in $X$;
    \item $R \in \complex_{\mathcal H}\cap X$ is a full-dimensional region and
    $
    f_\theta(x)=A_Rx+b_R$ for all $ x\in R;
    $
    \item $c\colon [k]\to \R^m$ satisfies
    $
    c(i)=c_\theta(\sigma)
   $ for every facet  $\sigma \in \complex_\theta^{d-1} \text{ with } \sigma\subseteq H_i.
    $
\end{enumerate}
\end{definition}

The following lemma shows that hyperplane representations completely characterize one-hidden-layer functions on the orthant.

\begin{restatable}{lemma}{lemhyperplanerepresentation}
\label{lem:hyperplane_representation}
    For every one-hidden-layer network $f_\theta \colon \orthant{d} \to \R^m$, there exists a hyperplane representation $(\mathcal{H}, A_R, b_R, c)$. 
    Moreover, this representation fully characterizes the function: 
    for any parameter $\eta$, we have $f_\eta = f_\theta$ if and only if $(\mathcal{H}, A_R, b_R, c)$ is also a hyperplane representation of $f_\eta$. 
\end{restatable}

\subsection{Explicit Description of the Fibers} 
\label{sec:explicit_description_layerwise_generic}

By \Cref{lem:hyperplane_representation}, a one-hidden-layer network $f_\theta$ is completely determined by its hyperplane representation $(\mathcal{H},A_R,b_R,c)$. Hence, to describe the fiber $\fiber{\theta}$, it suffices to characterize all parameters that realize this same hyperplane representation. In full generality, this leads to the semi-algebraic description given in Appendix~\ref{app:general-fiber}. Here, we restrict attention to generic and minimal parameters, where that description simplifies considerably.

We call a hidden neuron $i$ \emph{linear (on $\orthant{d}$)} if its preactivation is nonnegative on $\orthant{d}$, equivalently, in the one-hidden-layer case, if $\weights{1}_i \ge 0$ and $\bias{1}_i \ge 0$.
Let $(\mathcal{H},A_R,b_R,c)$ be the hyperplane representation of $f_\theta$, and write
$\mathcal{H}=\{H_1,\dots,H_k\}$. 
For each $i\in[k]$, choose $a_i\in\R^d$ and $t_i\in\R$ such that $\|a_i\|=1$ and
$H_i=\{x\in\R^d \mid \inner{a_i,x}+t_i=0\}$. 
We also fix a reference region $R$ of the arrangement and choose the signs of the 
$a_i$ so that
$R=\{x\in X \mid \inner{a_i,x}+t_i\le 0 \text{ for all } i\in[k]\}$. 

For a generic parameter, the hyperplane arrangement is generic. This implies that every nonlinear neuron corresponds to a unique hyperplane in $\mathcal H$ and that no cancellations occur. If the parameter is moreover minimal, then there are no dead neurons. Hence the hidden neurons split into exactly two classes: the $k$ nonlinear neurons corresponding to the breakpoint hyperplanes, and the remaining neurons, which are linear on $\orthant{d}$ We denote the latter index set by
$
J\coloneqq [n]\setminus [k]
$. 

Thus, in the generic minimal case, the fiber is governed by two types of choices. First, each nonlinear neuron may realize its associated hyperplane with either orientation, giving rise to a discrete family. Second, once these orientations are fixed, the resulting affine-linear discrepancy on the reference region $R$ must be realized by the linear neurons, producing a continuous family of factorizations.

To make this explicit, let $o\in\{1,-1\}^k$ be an orientation choice for the nonlinear neurons, and define
$S \coloneqq \{i\in[k]\mid o(i)=-1\}$, 
so that $S$ is the set of nonlinear neurons active on the reference region $R$. We write $S_\theta$ for the set of nonlinear neurons that are active on $R$ for the original parameter $\theta$. In the generic case, the gradient jump across $H_i$ is induced by a single neuron, so after normalizing so that $\|\weights{1}_i\|_2=1$, we have
$c(i)=\weights{2}_{:,i}$. 
If one changes the orientation of the nonlinear neurons from $S_\theta$ to $S$, then the linear neurons must compensate for the resulting change in the affine map on $R$. This is illustrated in \Cref{fig:sign_flip} (right panel). The required contribution from the linear neurons is given by 
\begin{equation}
\label{eq:generic_compensation}
    A_S \coloneqq A_R + \sum_{i \in S} c(i) a_i^\top,
    \qquad
    b_S \coloneqq b_R + \sum_{i \in S} c(i) t_i.
\end{equation}
since $\sum_{i \in S} -c(i)a_i^\top$ is the linear contribution of the nonlinear neurons that are active on $R$.
The semi-algebraic families introduced below naturally describe subsets of the full fiber $\fiber{\theta}$. However, they need not consist entirely of generic parameters in the sense of \Cref{def:generic_parameter}. Accordingly, the generic fiber
$
\gfiber{\theta}=\fiber{\theta}\cap \gparamspace{\architecture}
$
is obtained by intersecting these families with $\gparamspace{\architecture}$. 
\begin{restatable}{prop}{propgenericlayerwisefiber}
\label{prop:generic_layerwise_fiber}
Let $\theta$ be a generic and minimal parameter. 
Then the generic fiber is given by the disjoint union
$
\gfiber{\theta} \modulo{\sim} \ 
=
\bigcup_{S \subseteq [k]} \left(K_S \cap \gparamspace{\architecture}\right) \modulo{\sim},
$
where the union is taken over orientations $S \subseteq [k]$, and $K_S = \{p_S\} \times V_{S}$ is the semi-algebraic set defined as follows:
\begin{enumerate}[leftmargin=*] 
    \item $p_S$: 
    The parameters of the $k$ nonlinear neurons are fixed as follows. Let $o \in \{1,-1\}^k$ be the orientation determined by
    $
    S = \{i \in [k] \mid o(i)=-1\}.
    $
    Then
    \begin{equation}
        (\weights{1}_i, \bias{1}_i) = o(i)(a_i, t_i), \quad \weights{2}_{:,i} = c(i) \quad \forall i \in [k]. 
    \end{equation} 

    \item $V_{S}$: 
    The parameters of the linear neurons indexed by $J=[n]\setminus[k]$ and the output bias vector $\bias{2}$ are subject to the following semi-algebraic constraints, where $A_S$ and $b_S$ are defined in \eqref{eq:generic_compensation}:
    \begin{itemize}
        \item \emph{normalization and positivity}: $\|\weights{1}_i\|_2 = 1$ and $(\weights{1}_i, \bias{1}_i) \ge 0$ for all $i \in J$; 
        \item \emph{linear factorization}: $\weights{2}_{:,J} \weights{1}_J = A_S$; 
        \item \emph{bias alignment}: $\weights{2}_{:,J} \bias{1}_J + \bias{2} = b_S$. 
    \end{itemize}
\end{enumerate}
\end{restatable} 

In \Cref{subsec:dim_generic_fibres}, we establish the dimension of the semi-algebraic sets $V_S$ arising in the generic case based on the concept of the \emph{nonnegative column rank}.

\begin{restatable}{theorem}{thmlayerwisefibers}
\label{thm:layerwise-fibers}
Let $\theta$ be generic and minimal. 
For any given orientation $S\subseteq[k]$, consider the corresponding 
semi-algebraic component $V_S$ of the generic fiber defined in \Cref{prop:generic_layerwise_fiber}. 
Then: 
\begin{enumerate}[leftmargin=*]
    \item If $n-k < \min\{d,m\}$, then \(\dim(V_S) =(n-k)^2\)  if and only if $S = S_\theta$, and $V_S = \emptyset$ otherwise.

    \item If $n-k = d \leq m$, then \(\dim(V_S) =(n-k)^2\) for all $S \subseteq [k]$.  

    \item If $n-k = m < d$, then
 \(\dim(V_S) =(n-k)^2\) if and only if $\operatorname{cone}(A_S) \subseteq \col(A_S)$ is pointed, and $V_S = \emptyset$ otherwise, where $A_S$ is the compensation defined in \eqref{eq:generic_compensation}. 

    \item If $n-k = m+1 \leq d$, then \(\dim(V_S) = m^2 + m +d \) for all $S \subseteq [k]$. 
    \end{enumerate} 
\end{restatable}

\section{Symmetries from Layer Composition}
\label{sec:fibers_from_concatenation} 

We now analyze how interactions between layers give rise to symmetries that do not appear at the level of individual layers, particularly when certain neurons hide the geometric features of preceding layers. 
Throughout this section, let $\architecture = (d,n_1,n_2,m)$. 
These symmetries are complete for generic weights in the bottleneck case $(n_1 \leq d)$ as we show in \Cref{sec:bottleneck}.

This subsection formalizes the symmetries that arise when neurons in one layer hide the geometric features of a preceding layer. 
We show that this induces a coupled symmetry: 
a translation of a first-layer hyperplane 
can be exactly compensated by a corresponding shift in the second-layer biases.

\begin{definition}
Let $\theta \in \paramspace{\architecture}$. 
A neuron $i \in [n_2]$ \emph{hides} 
hyperplane $j\in[n_1]$ if $\weights{2}_{i,j} >0$, $\weights{2}_{i,k} \leq 0$ for all $k \neq j$, and $\bias{2}_i \leq 0$. 
The parameter $\theta$ \emph{hides} hyperplane $j$ if every neuron $i \in [n_2]$ does. 
\end{definition}

\paragraph{The Continuous Symmetry.}  

If the second layer hides hyperplane $j$, then translating the first-layer neuron $j$ by shifting $\bias{1}_j$, induces a change in $\activation{1}_j$ that can be 
compensated by a corresponding shift in the second-layer biases $\bias{2}$. 
This is because, under the hiding condition, second-layer neurons depend only on movement of the $j$-th hyperplane in a regime where $\activation{1}_j$ is linear, while negative weights and biases keep them inactive elsewhere. 
 See \Cref{fig:hidden_hyperplane} for an illustration. 
 For a parameter $\theta = (\weights{\ell},\bias{\ell})_{\ell \in [3]} \in \gparamspace{\architecture}$ 
 that hides hyperplane $j$, we define the set of translated parameters: 
 \begin{equation}
 \label{eq:translated_hyperplane}
     T_j(\theta) = \left\{
(\weights{1},
\bias{1}+t e_j,
\weights{2},
\bias{2}-t
{\weights{2}_{:j}}^\top
,
\weights{3},\bias{3})
\ \middle|\
t \geq \max_{i \in [n_2]} \frac{\bias{2}_i}{\weights{2}_{i j}}
\right\}. 
 \end{equation}

\begin{restatable}{lemma}{lemtranslatinghiddenhyperplaneinclusion}

\label{lem:translating_hidden_hyperplane_inclusion}
Let $\theta = (\weights{\ell},\bias{\ell})_{\ell \in [3]} \in \gparamspace{\architecture}$ be hiding hyperplane $j$. 
Then \(
T_j(\theta) \subseteq \fiber{\theta} \). 
\end{restatable}

\paragraph{The Discrete Symmetry.} 

While the local translation symmetry $T_j(\theta)$ exists for any hiding parameter, we identify an additional discrete symmetry that arises specifically when the second hidden layer consists of a single neuron, i.e., $n_2=1$. 
This symmetry permits a discrete ``flipping'' of the first-layer hyperplanes that leaves the realized function invariant.

More precisely, for $j \in [n_1]$ and any subset $I \subseteq \{1,\dots,n_1\} \setminus \{j\}$, define
\(
M_I = I_{n_1} - \sum_{k \in I} e_j e_k^\top - 2 \sum_{k \in I} e_k e_k^\top \in \R^{n_1 \times n_1},
\)
where $I_{n_1}$ is the identity matrix and $e_k$ is the $k$-th standard basis vector. 
The set
\(
\mathcal{G}_j = \{ M_I \mid I \subseteq \{1,\dots,n_1\} \setminus \{j\} \}
\)
is a finite Abelian group under matrix multiplication, with group operation
\(
M_I M_J = M_{I \triangle J}, \quad \text{for } I,J \subseteq \{1,\dots,n_1\}\setminus\{j\},
\)
where $I \triangle J$ denotes the symmetric difference of sets.  
Let $\theta = (\weights{\ell},\bias{\ell})_{\ell \in [3]} \in \gparamspace{\architecture}$ be generic and hiding hyperplane $j$, and suppose that $n_2 =1$. 
By the scaling symmetry, we may assume WLOG that $\weights{2}_{1k} \in \{+1,-1\}$ for all $k \in [n_1]$ and  
$\bias{2}_1 \in \{+1,-1\}$. We call such a parameter \emph{normalized}. 

\begin{restatable}{prop}{propflippinghiddenhyperplaneinclusion}
\label{prop:flipping_hidden_hyperplane_inclusion}
    Let $\theta = (\weights{\ell},\bias{\ell})_{\ell \in [3]} \in \paramspace{\architecture}$ be normalized and hiding hyperplane $j$.
For $M_I\in\mathcal{G}_j,$ let 
$M_I\cdot(\weights{1}, \bias{1}, \weights{2}, \bias{2}, \weights{3}, \bias{3})=(M_{I}\weights{1}, M_{I} \bias{1}, \weights{2}, \bias{2}, \weights{3}, \bias{3})$. 
    Then \( \mathcal{G}_j \cdot \theta \subseteq
\fiber{\theta} \). 
\end{restatable} 

\begin{figure}
    \centering
    \begin{subfigure}[t]{0.3\textwidth}
     \begin{tikzpicture}[scale=0.6]
        \small
\tikzset{
    bigarrow/.style={
        line width=1.5pt,
        draw,
        -{Latex[length=2mm,width=2mm]}
    }
}
\draw[orange!95!black,line width=1.5pt] (-3,0) -- (3,0) ;
\draw[orange!95!black,bigarrow] (2,0) -- (2,0.6);
\draw[orange!95!black,line width=1.5pt] (0,-3) -- (0,3) node [above] {$w_1x+b_1=0$};
\draw[orange!95!black,bigarrow] (0,1) -- (0.6,1);
\draw[cyan!75!blue,line width=1.5pt] (1,-3) --(1,0) -- (3,2);
\draw[cyan!75!blue,bigarrow] (1,-1) -- (1.6,-1);
\end{tikzpicture}
\end{subfigure}
        \begin{subfigure}[t]{0.3\textwidth}
     \begin{tikzpicture}[scale=0.6]
        \small 
\tikzset{
    bigarrow/.style={
        line width=1.5pt,
        draw,
        -{Latex[length=2mm,width=2mm]}
    }
}
\draw[orange!95!black,line width=1.5pt] (-3,0) -- (3,0);
\draw[orange!95!black,bigarrow] (2,0) -- (2,0.6);
\draw[orange!95!black,line width=1.5pt] (-1,-3) -- (-1,3) node [above] {$w_1x+b_1+s=0$};
\draw[dashed,orange!95!black,line width=1.5pt] (0,-3) -- (0,3);
\draw[orange!95!black,bigarrow] (0,1) -- (0.6,1);
\draw[orange!95!black,bigarrow] (-1,1) -- (-0.4,1);
\draw[cyan!75!blue,line width=1.5pt] (1,-3) --(1,0) -- (3,2);
\draw[cyan!75!blue,bigarrow] (1,-1) -- (1.6,-1);
\end{tikzpicture}
    \end{subfigure}
        \begin{subfigure}[t]{0.3\textwidth}
     \begin{tikzpicture}[scale=0.6]
        \small
        \tikzset{
    bigarrow/.style={
        line width=1.5pt,
        draw,
        -{Latex[length=2mm,width=2mm]}
    }
}
\draw[orange!95!black,line width=1.5pt] (-3,0) -- (3,0);
\draw[orange!95!black,bigarrow] (2,0) -- (2,-0.6);
\draw[orange!95!black,line width=1.5pt] (-2,-2) -- (3,3) node [above,  xshift=-10mm] {$(w_1-w_2)x+b_1-b_2=0$};
\draw[orange!95!black,bigarrow] (1,1) -- (1.4,0.6);
\draw[cyan!75!blue,line width=1.5pt] (1,-3) --(1,0) -- (3,2);
\draw[cyan!75!blue,bigarrow] (1,-1) -- (1.6,-1);

\end{tikzpicture}
    \end{subfigure}
\caption{Parameters in the fiber when there is a hidden hyperplane. 
Two first-layer hyperplanes are shown in orange and one second-layer bent hyperplane is shown in blue. 
The left two panels illustrate the continuous symmetry and the one on the right illustrates the discrete symmetry.}
\label{fig:hidden_hyperplane}
\end{figure}
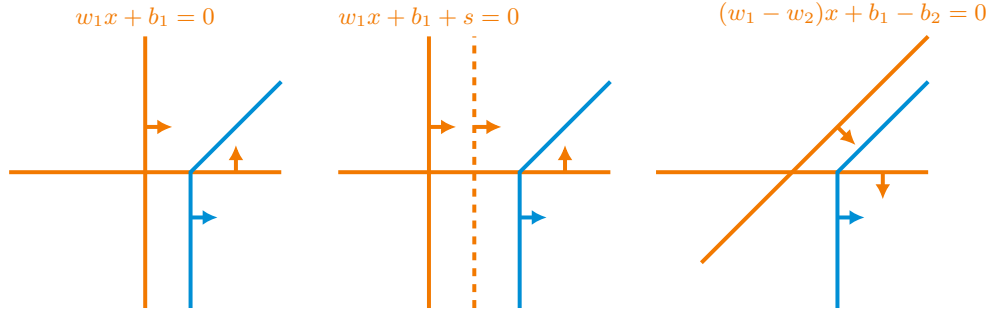

\section{Three-Layer Bottleneck Fibers}
\label{sec:bottleneck}

We now give a complete description of generic fibers for three-layer bottleneck architectures. Throughout, let
$\architecture=(d,n_1,n_2,m)$ with $n_1\le d$.  
For such architectures, the first-layer image is generically the nonnegative orthant, so the fiber structure is governed by  
the action of the second layer on this. 
Our analysis splits into two cases. 
If no first-layer hyperplane is hidden, 
the first layer is rigid along the fiber, reducing the problem to the one-hidden-layer classification in \Cref{sec:layerwise_fibers}. 
If a hyperplane is hidden, the translation and sign-flipping symmetries from \Cref{sec:fibers_from_concatenation} account for the remaining generic degrees of freedom. 
We show that, in the bottleneck regime, these symmetries are complete.

\paragraph{Non-Hiding Parameters.}
If every first-layer hyperplane remains functionally visible through the second layer, then the first hidden layer is fixed across the fiber up to trivial symmetries.

\begin{prop}
\label{prop:3layer_layerwise_reduction_main}
Let $\theta = (\weights{\ell}, \bias{\ell})_{\ell \in [3]} \in \gparamspace{\architecture}$ have no dead neurons and be non-hiding. Then the first-layer parameters are fixed across $\gfiber{\theta}$ up to permutation and positive scaling, and the generic fiber reduces to the one-hidden-layer fiber of the induced map
$
g_{\theta'} \colon \orthant{n_1} \to \R^m,
$ given by $
g_{\theta'}(y)=\weights{3}[\weights{2}y+\bias{2}]_+ + \bias{3}.
$
In particular,
$
\gfiber{\theta}\modulo{\sim}
\cong
\gfiber{\theta'}\modulo{\sim},
$
where the right-hand side is described by \Cref{prop:generic_layerwise_fiber}.
\end{prop}

\begin{proof}[Proof sketch]
The rigidity of the first layer follows from \Cref{thm:intersected_hyperplanes_stay}. Since $n_1\le d$ and $\weights{1}$ has full rank, the first hidden layer maps surjectively onto $\orthant{n_1}$, so equality of the realized functions reduces to equality of the induced one-hidden-layer maps on the orthant which is described in \Cref{sec:layerwise_fibers}. The detailed proof is given in \Cref{app:fibers_non_hiding}.
\end{proof}

\paragraph{Hiding Parameters.}

We now turn to the case where a first-layer hyperplane is hidden by the second layer. 
If the second layer has width at least two, the hidden hyperplane may be translated, but its direction remains rigid.

\begin{restatable}{prop}{fiberfromhiding}
\label{prop:fiber_from_hiding}
Let $\theta = (\weights{\ell},\bias{\ell})_{\ell \in [3]} \in \gparamspace{\architecture}$ be generic and hiding hyperplane $j$, and assume that $n_2 \geq 2$. 
Then $
\gfiber{\theta} \modulo{\sim}\  = (T_j(\theta)\cap \gparamspace{\architecture}) \modulo{\sim}
$.   
\end{restatable}

If the second hidden layer consists of a single neuron, then in addition to the translation symmetry there is a finite family of discrete sign-flipping symmetries. 
\begin{restatable}{prop}{fiberfromhidingsingleneuron}
\label{prop:hiding_fiber_single_second_layer}
Let $\theta = (\weights{\ell,\theta},\bias{\ell,\theta})_{\ell \in [3]} \in \gparamspace{\architecture}$ be generic and normalized, assume that $n_2=1$, and suppose that $\theta$ hides hyperplane $j$. Then
$
\gfiber{\theta}\modulo{\sim}
=
\left(
\bigcup_{I \subseteq [n_1]\setminus \{j\}}
\left(T_j(M_I \cdot \theta)\cap \gparamspace{\architecture}\right)
\right)\modulo{\sim}
$. 
\end{restatable}

The proofs of both propositions are deferred to \Cref{app:fibers_hiding}. Geometrically, they rely on two facts. First, all first-layer hyperplanes except the hidden one remain visible along the fiber and are therefore fixed. Second, in the case $n_2\ge 2$, each second-layer neuron produces, on the region where only neuron $j$ is active in the first layer, a visible facet parallel to the hidden hyperplane. Since there are at least two such parallel facets in the breakpoint complex, their common direction determines the direction of the hidden hyperplane, so the only remaining freedom is to translate it. By contrast, when $n_2=1$, there is only one second-layer bent hyperplane. Its normal vectors can then be realized by several different sign patterns of the first layer, and these alternative realizations give rise to the discrete sign-flipping symmetries described by the matrices $M_I$.

\paragraph{Characterization of Identifiability.}

By combining non-hiding and hiding cases, we obtain a complete classification of identifiability  for three-layer bottleneck networks with generic parameters based entirely on sign patterns of second-layer weights and biases. 

\begin{theorem} 
\label{thm:charaterization_identifiablity_bottleneck}
    Let $\architecture=(d,n_1,n_2,m)$ with 
    $d\geq n_1 \geq 2$. 
    Then a generic parameter $\theta=(\weights{\ell},\bias{\ell})_{\ell \in [3]} \in \gparamspace{\architecture}$ is non-identifiable 
    if and only if at least one of the following holds: 
\begin{enumerate}[leftmargin=*]
    \item There exists a row of $[\weights{2},\bias{2}]$ with all entries negative.
    
    \item There exists a row of $[\weights{2},\bias{2}]$ with all entries positive.
    
    \item There exists $j \in [n_1]$ such that 
   $[\weights{2},\bias{2}]$ 
   is 
   positive in column $j$ and negative elsewhere. 
    
\end{enumerate}
\end{theorem}

\begin{proof}
Conditions (1) produces a dead neuron, and (2) produces a neuron that acts linearly, leading to a positive-dimensional fiber by \Cref{thm:layerwise-fibers}. 
Condition (3) induces the hiding symmetries $T_j$ and $\mathcal{G}_j$ which are distinct from the permutation and scaling symmetry. 
If neither holds, then by \Cref{prop:3layer_layerwise_reduction_main} the first layer is rigid and the second-layer fiber is trivial, implying identifiability. 
\end{proof}
Since this yields a complete characterization and containment in the explicit semi-algebraic sets can be evaluated in polynomial time, we obtain the following.

\begin{restatable}{prop}{polytimeequivalencebottleneck}
\label{prop:polytime_equivalence_bottleneck}
There exists a polynomial-time algorithm that, given a bottleneck architecture
$\architecture=(d,n_1,n_2,m)$ with $n_1\le d$ and two generic parameters
$\theta,\eta\in\gparamspace{\architecture}$, decides whether
$
f_\theta=f_\eta.
$
\end{restatable}

The relative volume of the set of identifiable parameters can be computed from Theorem~\ref{thm:charaterization_identifiablity_bottleneck}. 
Simple lower and upper bounds can be given as follows: 

\begin{restatable}{cor}{corvolumebound}
\label{cor:volume-bound}
Let $\architecture=(d,n_1,n_2,m)$ with 
$d\geq n_1 \geq 2$, 
let 
\(
\mathcal{I} \subseteq \paramspace{\architecture}
\)
be the set of identifiable parameters, 
and $\mathbb{B}_r$ the radius-$r$ ball in $\paramspace{\architecture}$ for an arbitrary $r>0$. 
Then  
\[
1 - n_2 \, 2^{-n_1}  - n_1 \, 2^{-n_2n_1}
\;\le\;
\frac{\operatorname{vol}(\mathcal{I} \cap \mathbb{B}_r)}{\operatorname{vol}(\paramspace{\architecture} \cap \mathbb{B}_r)}
\;\le\;
(1 - 2^{-(n_1+1)})^{n_2}. 
\]
\end{restatable}

\section{Implications for Deeper Networks}
\label{sec:implications_for_deep_NNs}

This section discusses implications of our fiber characterization for deep architectures.

\paragraph{Fibers are not Localizable.} 

A natural question is whether identifiability of the full parameter can be reduced to layerwise identifiability. 
For deep polynomial networks, identifiability can indeed be established by examining pairwise compositions \citep{usevich2025identifiabilitydeeppolynomialneural}. 
The following proposition shows that such a reduction is not possible for ReLU networks. 
Specifically, identifiability in ReLU networks can depend on interactions between the range of earlier layers and the non-linearities of later ones. 

\begin{restatable}{prop}{proplocalizable}
\label{prop:localizable}
    Let $\architecture=(n_0,n_1,n_2,n_3,m)$ with 
    $n_0\geq n_1\geq n_2$.  
    Then there exists $\theta = (\weights{\ell},\bias{\ell})_{\ell \in [4]}$ such that, for the layer maps $f_\ell \colon \R^{n_{\ell-1}} \to \R^{n_\ell}$ defined by $x \mapsto \left[\weights{\ell}x+\bias{\ell}\right]_+$, the compositions $f_2 \circ f_1$ and $f_3 \circ f_2$ are identifiable, 
    while $f_3 \circ f_2 \circ f_1$ is not. 
\end{restatable}

\paragraph{Locally Conserved Quantities under Gradient Flow.} 

We leverage the newly found symmetries on   orthants of the parameter space to identify conserved quantities 
under gradient flow. 
\cite{marcotte2023abide} identify quantities 
conserved over the entire parameter space and show that, under suitable conditions, their characterization captures all independent conserved functions. 
Our goal is to identify distinct functions by relaxing the requirement of global conservation, 
instead seeking quantities that remain constant only on subsets of the parameter space. 
Indeed, the 
layerwise continuous symmetries in \Cref{sec:layerwise_fibers}
which are induced by the monoid action of $\GL^+$, give rise to locally conserved quantities.

\begin{theorem}\label{th:locally_conserved}
    Let $O\subseteq \paramspace{\architecture}$ be an orthant of parameter space where $\ell$-layer neurons $J \subseteq [n_\ell]$ are linear. 
    Then  $(\weights{\ell}_{:,J})^\top\weights{\ell}_{:,J}  - \weights{\ell-1}_{J}(\weights{\ell-1}_{J})^\top-\bias{\ell-1}_J (\bias{\ell-1}_J)^\top$ is a conserved quantity on $O$. 
\end{theorem}

A proof can be found in \Cref{subsec:conserved_quantities}; moreover we show in 
\Cref{app:conserved_layer_concat} that the symmetries discovered in \Cref{sec:fibers_from_concatenation} do not induce locally conserved quantities.

\section{Conclusion, Limitations, and Outlook}

We developed a framework for analyzing parameter symmetries in deep ReLU networks beyond classical scaling and permutation invariances. 
We provided a complete description of the symmetries and the parameter fibers for three-layer bottleneck architectures with generic parameters, including explicit formulas for their dimensions and a polynomial-time algorithm for deciding functional equivalence of parameters. 
Our results show that parameter redundancy in deep ReLU networks arise not only from layerwise effects but also from compositional interactions across layers. 
A natural next step is to investigate the topology of fibers, including 
the number of connected components. 
While the layerwise and hiding symmetries described in this paper are not specific to the generic three-layer bottleneck setting, 
establishing their completeness and obtaining explicit fiber descriptions beyond this regime is likely substantially more involved. 
In wider architectures $(n_1>d)$, later-layer hyperplanes must be analyzed relative to a union of polyhedral images rather than the non-negative orthant, 
suggesting an exponential blow-up in the number of relevant activation patterns. 
In deeper networks, equivalence cannot be reduced to pairwise layer composition (\Cref{prop:localizable}), and one must track how the image of an entire prefix is intersected by the hyperplanes of the subsequent layers. 
In nongeneric settings, the geometric rigidity underlying our analysis breaks down, leading to additional degeneracies and symmetries.  

These considerations indicate that the bottleneck and genericity assumptions in our main results isolate a regime in which genuinely deep compositional symmetries already arise while still admitting explicit and efficiently checkable descriptions. 
This perspective is consistent with complexity-theoretic evidence suggesting that our result lie near a tractability boundary for deciding functional equivalence and for obtaining explicit descriptions of ReLU network fibers. We discuss these extensions and their relation to hardness barriers in \Cref{app:discussion_beyond_bottleneck}.


 \subsection*{Acknowledgments} 
 This project has been supported by the
 Deutsche Forschungsgemeinschaft (DFG, German Research Foundation)  project 464109215 
 within the priority programme SPP 2298 ``Theoretical Foundations of Deep Learning''.  
 JMG was supported by the DFF Sapere Aude Starting Grant ``GADL''. 
 GM was partially supported by 
DARPA AIQ grant HR00112520014, 
NSF grants 
DMS-2522495, 
DMS-2145630, 
CCF-2212520, 
and BMFTR in DAAD project 57616814 (SECAI).

\bibliographystyle{unsrtnat}
\bibliography{ref}
\newpage

\appendix

\listofappendices

\section{Background on Polyhedral Geometry}
\label{app:polyheadral_geometry}
\addcontentsline{apx}{section}{\protect\numberline{\thesection}Background on Polyhedral Geometry}

For a vector $a \in \R^{d}$ and a scalar $b \in \R$, the \emph{hyperplane} $H \coloneqq \{x \in \R^d \mid \inner{a,x}+b=0\}$ subdivides $\R^d$ into \emph{half-spaces} $H^{+}\coloneqq \{x \in \R^d \mid \inner{a,x}+b\geq 0\} \text{ and }{H}^{-} \coloneqq \{x \in \R^d \mid \inner{a,x}+b\leq0\}.$
A finite set of hyperplanes $\mathcal H = \{H_1,\dots,H_n\}$ is called a \emph{hyperplane arrangement}. 
A hyperplane arrangement $\mathcal H$ in $\R^d$ is \emph{generic} if the intersection of any subset of $k \leq d$ hyperplanes has codimension $k$, and no $d+1$ hyperplanes have a common intersection. 
Given an arbitrary arrangement $\mathcal H$, let
$
L \coloneqq \bigcap_{H \in \mathcal H} \lin(H)
$
be the maximal linear subspace contained in all hyperplanes.
Projecting $\R^d$ orthogonally onto $L^\perp$ yields an induced arrangement
$\mathcal H^{\mathrm{ess}}$ in $L^\perp$, called the \emph{essentialization} of
$\mathcal H$.

A \emph{polyhedral complex} $\complex$ is a finite collection of polyhedra such that 
\begin{enumerate}
    \item $\emptyset \in \complex$,
    \item if $P \in \complex$, then all faces of $P$ are in $\complex$, and 
    \item if $P, P' \in \complex$ and $P\cap P'\neq\emptyset$, then $P \cap P'$ is a face of both $P$ and $P'$. 
\end{enumerate} 
For a polyhedral complex $\complex$ in $\R^d$ and $k \leq d$ we denote by $\complex^k$ the set of $k$-dimensional polyhedra in~$\complex$. 
We call $\complex^d$ the \emph{regions}, $\complex^{d-1}$ the \emph{facets}, and $\complex^{d-2}$ the \emph{ridges} of $\complex$. 
The \emph{dimension} of a complex~$\complex$ is the maximal dimension of its polyhedra. 
Given a face $\sigma \in \complex$, we denote by $\aff(\sigma) \subseteq \R^d$ the unique smallest affine subspace containing $\sigma$. The \emph{relative interior} of $\sigma$ is the interior of $\sigma$ within the affine space $\aff(\sigma)$.

Let $\complex$ be a polyhedral complex in $\R^d$ and let $\tau \in \complex$ be a face.
The \emph{star} of $\tau$ is
\(
\str_\complex(\tau) \coloneqq \{\, \sigma \in \complex \mid \tau \subseteq \sigma \,\}.
\)
When $\complex$ is clear from the context, we omit the subscript and write $\str(\tau)$. For any $k\leq d$ and any faces $\tau \in \complex^{k-1},\sigma \in \complex^{k}$ with $\tau \subseteq \sigma$, we let $e_{\sigma/\tau} \in \R^d$ denote the normal vector of $\tau$ relative to $\sigma$, defined as the unique unit vector that lies in $\aff(\sigma)$, is orthogonal to $\aff(\tau)$, and points from the relative interior of $\tau$ into the relative interior of $\sigma$. 
For a subset  $S \subseteq \complex$, we denote the \emph{support} by $|S| \coloneqq \bigcup_{P \in S} P$ and by $\# S$ the number of elements contained in $S$.

A function $f\colon\R^d\to\R^m$ is \emph{continuous and piecewise linear} (CPWL), if there exists a complete polyhedral complex $\complex$ such that the restriction of $f$ to each polyhedron $P\in\complex$ is an affine linear function. 
If this condition is satisfied, we say that $f$ and $\complex$ are \emph{compatible} with each other. 
A vector $x \in  \R^d$ is a \emph{breakpoint} of $f$ if there is no open set $U \subseteq \R^d$ containing $x$ such that $f$ is affine linear on $U$. 
For a CPWL function $f\colon \R^d \to \R^m$, let $\breakpoints{f} $ be the set of breakpoints of $f$. 

A pair $(\complex,c)$ forms a \emph{balanced (weighted) polyhedral complex} if the weight function satisfies the \emph{balancing condition} at every  $\tau \in \complex^{d-2}$: 
\begin{equation*}\label{eq:balancing-condition}
    \sum\limits_{\sigma \in \str(\tau)^{d-1}} c(\sigma)\cdot e_{\sigma/\tau}^\top=0.
\end{equation*}
If $\complex$ is a polyhedral fan in $\R^2$, then its unique $0$-dimensional face is the origin. The balancing condition then requires that, for each $i\in[m]$, the weighted sum of the unit generators of the rays in $\complex$, $\sum\limits_{\sigma \in \complex^{1}} c_i(\sigma)\cdot e_{\sigma}$, equals zero. 
In higher dimensions, taking the star of a codimension-$2$ face $\tau$ and modding out $\tau$ yields a two-dimensional fan. 
Intuitively, the balancing condition requires that this two-dimensional fan is balanced in the same sense as described above. 
 
 The following correspondence between CPWL functions and balanced complex follows from Tropical Geometry (\citealp[Proposition 3.3.10]{maclagan2015introduction}; \citealp[Proposition 3.3.2]{maclagan2015introduction}), as \cite{brandenburg2025decomposition} outlines for the scalar-valued case.

\begin{lemma}
\label{lem:APP:structure_theorem_tropical_geometry}
   Let $f \colon \R^d \to \R^m$ be a CPWL function compatible with a polyhedral complex $\complex$. 
   For a facet $\sigma \in \complex^{d-1}$, let $P,Q \in \complex^d$ be the unique polyhedra such that $P \cap Q = \sigma$, 
   and 
   suppose that $f(x) = A_Px + b_P$ for all $x \in P$ and $f(x) = A_Qx + b_Q$ for all $x \in Q$. 
   Then 
   \[
   c_f(\sigma)  \coloneqq A_{P} e_{P/\sigma} +  A_{Q} e_{Q/\sigma} =  (A_{P} - A_{Q})e_{P/\sigma} 
   \]
 defines a weight function $c_f$ such that $(\complex,c_f)$ is a balanced polyhedral complex. Conversely,  for  a balanced complex $(\complex,c)$, there exists a unique function $f$, up to addition of a global affine linear function, that is compatible with $\complex$ and satisfies $c_f=c$.
\end{lemma}

\section{Polyhedral Geometry of ReLU Networks}
\label{app:relu_geometry}
\addcontentsline{apx}{section}{\protect\numberline{\thesection}Polyhedral Geometry of ReLU Networks} 

This appendix collects additional geometric facts about ReLU networks that are used in
Section~\ref{sec:polyhedral_geometry_relu}. Our goal here is not to redevelop the full framework of
\cite{rolnick2020reverse,grigsby2023hiddensymmetriesrelunetworks} and \cite{grillo2026relunetworksadmitidentifiable}, but only to state the specific tools
from those works that will be needed for our visibility result in \Cref{thm:intersected_hyperplanes_stay}.

\subsection{Canonical Polyhedral Complexes and Bent Hyperplanes} 
\label{app:canonical_complex}

We begin by recalling the canonical polyhedral complex associated with a ReLU network as introduced in \cite{doi:10.1137/20M1368902}.
For a single ReLU layer
$
\phi_{W,b}(x)=[Wx+b]_+ \colon \R^d\to\R^m,
$
the breakpoints are contained in the hyperplane arrangement
$
\mathcal{H}_{W,b}
=
\bigl\{
\{x\in\R^d \mid W_ix+b_i=0\}
\bigr\}_{i=1}^m.
$
Each sign pattern $\mathbf{s}=(s_1,\dots,s_m)\in\{+,0,-\}^m$ defines a polyhedron
$
P_{\mathbf{s}}
\coloneqq
\bigcap_{i=1}^m H^{s_i}_{W_i,b_i},
$
where $H^+$ and $H^-$ denote the two closed halfspaces of the hyperplane
$\{x\mid W_ix+b_i=0\}$ and $H^0$ denotes the hyperplane itself.
The collection of all such polyhedra forms the \emph{canonical polyhedral complex}
$\complex_{W,b}$ of the layer $\phi_{W,b}$.
On each cell of $\complex_{W,b}$, the map $\phi_{W,b}$ is affine linear.

Now let $\theta\in\paramspace{\architecture}$ for an architecture
$
\architecture=(n_0,\dots,n_L,n_{L+1}).
$
The canonical polyhedral complex of $f_\theta$ is constructed iteratively.
Set
$
\complex_{\theta,0}\coloneqq \R^{n_0}.
$
Suppose that $\complex_{\theta,\ell-1}$ has already been defined and that all preactivations up to
layer $\ell$ are affine linear on each polyhedron $R\in\complex_{\theta,\ell-1}$.
For each neuron $(j,\ell)$, the restriction $\preactivation{\ell,\theta}_j|_R$ is an affine function on $R$.

If $\preactivation{\ell,\theta}_j|_R$ is non-constant, its zero set defines a hyperplane in
$\aff(R)$, which we denote by
$
H_R(\ell,j)
\coloneqq
\{x\in \aff(R)\mid \preactivation{\ell,\theta}_j(x)=0\}.
$
If $\preactivation{\ell,\theta}_j|_R$ is constant, then it introduces no breakpoint on $R$ and does
not affect the subdivision.
Refining every cell $R\in\complex_{\theta,\ell-1}$ by the hyperplanes
$H_R(\ell,1),\dots,H_R(\ell,n_\ell)$ yields the next complex $\complex_{\theta,\ell}$.
We define the \emph{canonical polyhedral complex} of the network by
$
\complex_\theta \coloneqq \complex_{\theta,L}.
$

The polyhedra in $\complex_\theta$ are indexed by global activation patterns
$
\mathbf{s}\in \{+,0,-\}^{n_1}\times \cdots \times \{+,0,-\}^{n_L},
$
and on each such cell the realized map $f_\theta$ is affine linear.

For a neuron $(j,\ell)$, the associated \emph{bent hyperplane} is
\[
\bhyperplane{\ell,j}(\theta)
\coloneqq
\bigcup_{\substack{R\in\complex_{\theta,\ell-1}\\ \preactivation{\ell,\theta}_j|_R \text{ non-constant}}}
\{x\in R \mid \preactivation{\ell,\theta}_j(x)=0\}.
\]
Thus $\bhyperplane{\ell,j}(\theta)$ is the locus where the neuron changes its linear behavior.
We write
$
\hyperplanes{\ell}{\theta}
\coloneqq
\{\bhyperplane{\ell,j}(\theta)\mid j\in[n_\ell]\}
$
for the collection of bent hyperplanes in layer $\ell$.

\subsection{Generic Parameters} 
\label{app:generic-parameters}

We now record the precise notion of genericity used throughout the paper. This definition is close to those adopted in the works of \citet{phuong2020functional,grillo2026relunetworksadmitidentifiable,marissa}. 

\begin{definition}
\label{def:generic_parameter}
A parameter $\theta \in \paramspace{\architecture}$ is called \emph{generic} if it satisfies the
following two conditions:
\begin{enumerate}[leftmargin=*]
    \item {Supertransversality:} every face $\tau \in \complex_\theta^{d-k}$ is contained in
    exactly $k$ bent hyperplanes;

    \item {Maximal rank of masked products:} for every $\ell \in [L+1]$ and every sequence of
    index sets $S_i \subseteq [n_i]$ for $i \le \ell-1$, the matrix product
    $
    \weights{\ell} D_{S_{\ell-1}} \weights{\ell-1}\cdots D_{S_1}\weights{1}
    $
    attains the maximum possible rank, namely
    $
    \min\Bigl(d,\; n_\ell,\; \min_{1\le i\le \ell-1}|S_i|\Bigr).
    $
\end{enumerate}
We denote the set of generic parameters by
$
\gparamspace{\architecture}\subseteq\paramspace{\architecture}.
$
\end{definition}

\begin{lemma}
\label{lem:generic_open_dense}
The set $\gparamspace{\architecture}$ is open and dense in $\paramspace{\architecture}$.
\end{lemma}

\subsection{Breakpoint Complexes and Visible Facets}
\label{app:breakpoint_complex}

Not every facet of the canonical polyhedral complex contributes a genuine breakpoint of the realized
function. We therefore distinguish the full canonical subdivision from the visible part of the
network.

For a parameter $\theta\in\paramspace{\architecture}$, let
$
\breakpoints{f_\theta}
\subseteq \R^{n_0}
$
denote the breakpoint set of the realized function.
The \emph{breakpoint complex} is the subcomplex
$
\bcomplex{\theta}
\coloneqq
\{P\in\complex_\theta \mid P\subseteq \breakpoints{f_\theta}\}.
$
Its support is exactly the breakpoint set:
$
|\bcomplex{\theta}|=\breakpoints{f_\theta}.
$
Since $f_\theta$ is compatible with the canonical polyhedral complex $\complex_\theta$, it induces a
weight function
$
c_\theta \colon \complex_\theta^{d-1}\to\R^m
$
as in \Cref{lem:structure_theorem_tropical_geometry}. A facet of $\complex_\theta$ is visible if and
only if the corresponding weight is nonzero, that is,
$
\bcomplex{\theta}^{d-1}
=
\{\sigma\in\complex_\theta^{d-1}\mid c_\theta(\sigma)\neq 0\}
$. 
Thus the breakpoint complex can be viewed as the visible part of the canonical polyhedral complex,
equipped with the restricted weight function $c_\theta$. 

An important observation is that, for generic parameters, the breakpoint complex does not depend on the specific parameter choice but just on the realized function: 

 \begin{prop}  [\citealp{grillo2026relunetworksadmitidentifiable}]
\label{prop:canonicalbcomplex}
If $\theta,\eta$ are generic and satisfy $f_\theta=f_\eta$, then
$\bcomplex{\theta}=\bcomplex{\eta}.$
\end{prop}

\subsection{One-Hidden-Layer Hyperplane Weights}
\label{app:one_hidden_layer_weights}

Compared to deep networks, ReLU networks with a single hidden layer have considerably simpler breakpoint geometry: 
the visible breakpoints lie on a hyperplane arrangement, and the weight function is constant along each
visible hyperplane. 
This structure underlies the hyperplane representation introduced in Section~\ref{sec:layerwise_fibers}. 

The following result is taken from \citet{grillo2026relunetworksadmitidentifiable} and is similar to a result of \cite{PetzkaTS20}. 

\begin{prop}
\label{prop:tropical_weight_1layer}
Let $\theta=(\weights{1},\bias{1},\weights{2},\bias{2})$ be the parameter of a one-hidden-layer
network, and let $\sigma\in\complex_\theta^{d-1}$ be a facet. Then
$
c_{f_\theta}(\sigma)
=
\sum_{i\in I}\weights{2}_{:,i}\,\|\weights{1}_i\|$ with $
I=\{i\in[n_1]\mid H_i=\aff(\sigma)\},
$
where
$
H_i=\{x\in\R^d \mid  \weights{1}_ix+\bias{1}_i=0\}.
$
\end{prop}

As an immediate consequence, if $\sigma,\sigma' \in \complex_\theta^{d-1}$ are visible facets
contained in the same breakpoint hyperplane $H_i$, then
$c_\theta(\sigma)=c_\theta(\sigma')$. 
Thus, each visible breakpoint hyperplane carries a well-defined vector weight.

\subsection{Bending and Visibility Across Layers}
\label{app:bending_visibility}

The notions introduced in this subsection have been developed systematically in prior work on the polyhedral geometry of ReLU networks \citep[see especially][]{phuong2020functional,rolnick2020reverse,grigsby2023hiddensymmetriesrelunetworks,grillo2026relunetworksadmitidentifiable}. 

We recall here several local geometric criteria from  \citet{grillo2026relunetworksadmitidentifiable} that we will use for our analysis of visibility in \Cref{sec:last_layer}. 

\begin{definition}
Let $(\complex,c)$ be a weighted polyhedral complex in $\R^d$, and let $\tau \in \complex^{d-2}$ be a ridge.
We say that a facet $\sigma \in \str(\tau)^{d-1}$ is \emph{non-bending at $\tau$} if there exists another
facet $\sigma' \in \str(\tau)^{d-1}$ such that
$
e_{\sigma/\tau}=-e_{\sigma'/\tau}
$ and $
c(\sigma)=c(\sigma').
$
Otherwise, we say that $\sigma$ is \emph{bending at $\tau$}.
The ridge $\tau$ is called \emph{bending} if at least one adjacent facet is bending at $\tau$,
and \emph{non-bending} otherwise.
\end{definition}

The next observation states that ridges arising entirely from a single layer are non-bending.
\begin{lemma}
\label{lem:hyperplaneAimpliesnonbending}
If $\tau \in \bcomplex{\theta}$ lies only on bent hyperplanes from a single layer $\ell \in [L]$,
then $\tau$ is non-bending.
\end{lemma}

Thus any bending ridge must arise from the interaction of different layers. The converse need not hold
in general, so we isolate the following property.

\begin{definition}
We call a parameter $\theta$ \emph{honest} if every non-bending ridge of $\bcomplex{\theta}$ lies only
on bent hyperplanes from a single layer.
\end{definition}

We also need a visibility condition excluding cancellations of facets.

\begin{definition}
We call a parameter $\theta$ \emph{cancellation-free} if for every facet
$\sigma \in \complex_\theta^{d-1}$ one has
$
\sigma \notin \bcomplex{\theta}
\quad\Longleftrightarrow\quad
\text{there exists a layer } k \in [L] \text{ such that } \mathbf{s}_{k,j}(\sigma)=- \text{ for all } j\in[n_k].
$
\end{definition}

For the visibility theorem, we only need these properties in the generic case.

\begin{lemma}

\label{lem:generic_honest}
Generic parameters are honest and cancellation-free.
\end{lemma}

We next recall the explicit form of the weight function across a visible facet.

\begin{prop}
\label{prop:tropweightofNNs}
Let $\theta$ be supertransversal, let $\sigma \in \complex_\theta^{d-1}$ be a facet, and let
$B_{\ell,i}$ be the unique bent hyperplane containing $\sigma$. If $S_k \subseteq [n_k]$ denotes the
set of strictly active neurons on the relative interior of $\sigma$ in layer $k$, then the gradient jump is given by
\[
c_\theta(\sigma)
=
\left\|
(\weights{1})^\top D_{S_1}\cdots (\weights{\ell-1})^\top D_{S_{\ell-1}}(\weights{\ell})^\top e_i
\right\|
\;
\weights{L+1}D_{S_L}\cdots \weights{\ell+1}e_i.
\]
\end{prop}

Finally, the following lemma identifies, at a bending ridge, which adjacent facets come from the earlier
layer.

\begin{lemma}
\label{lem:bendingoptions}
Let $\theta$ be supertransversal and cancellation-free.
Let $\tau \in \bcomplex{\theta}^{d-2}$ be a bending ridge, and let
$
H \in \hyperplanes{k}{\theta},
$ and $
B \in \hyperplanes{\ell}{\theta},
$ with $
k<\ell,
$
be the unique bent hyperplanes containing $\tau$.
Let $R \in \complex_{\theta,k-1}^d$ be the cell with $\relint(\tau)\subseteq \relint(R)$.
Then:
\begin{enumerate}[leftmargin=*]
    \item If $\#\str_{\bcomplex{\theta}}(\tau)^{d-1}=4$, then there are exactly two facets
    $\sigma_1,\sigma_2 \in \str_{\bcomplex{\theta}}(\tau)^{d-1}$ such that
    $
    e_{\sigma_1/\tau}=-e_{\sigma_2/\tau},
    $
    and these are precisely the facets contained in $H$. In particular,
    $
    H\cap R = \aff(\sigma_1)\cap R.
    $

    \item If $\#\str_{\bcomplex{\theta}}(\tau)^{d-1}=3$, then there is a unique facet
    $\sigma \in \str_{\bcomplex{\theta}}(\tau)^{d-1}$ that is not adjacent to a region on which
    $f_\theta$ is constant, and this facet is the one contained in $H$. In particular,
    $
    H\cap R = \aff(\sigma)\cap R.
    $
\end{enumerate}
\end{lemma}

\subsection{Visibility of First-Layer Hyperplanes in Three-Layer Networks}
\label{sec:last_layer}

We now prove the visibility result stated as \Cref{thm:intersected_hyperplanes_stay} in the main text.
The key point is that bent hyperplanes from the last hidden layer are always visible in the realized
function, and that a first-layer hyperplane becomes rigid along the fiber as soon as it is anchored
by the second layer.

We begin with a general observation about sums of CPWL functions.

\begin{lemma}
\label{lem:unionofbreakpoints}
Let $f,g \colon \R^d \to \R^m$ be CPWL functions, and let
$B_1 \subseteq \breakpoints{f}$ be a subset of the breakpoint set that is a pure polyhedral complex of codimension $1$. Suppose that $B_1 \cap \breakpoints{g}$ has codimension at least $2$. 
Then
\[
B_1  \subseteq \breakpoints{f+g}.
\]
\end{lemma}

\begin{proof}
We first show that
\(
B_1 \setminus \breakpoints{g} \subseteq \breakpoints{f+g}
\).
Let $x \in B_1 \setminus \breakpoints{g}$.
Because $g$ is a CPWL function and $x \notin \breakpoints{g}$, the function $g$ is affine-linear in a neighborhood of $x$.
Because $f$ is not affine-linear at $x$ (since $x \in B_1 \subseteq \breakpoints{f}$), the sum $f+g$ is also not affine-linear at $x$.
Hence $x \in \breakpoints{f+g}$.
Thus,
\(
B_1 \setminus  \breakpoints{g}) \subseteq \breakpoints{f+g}.
\)
Because $B_1$ is the support of a pure complex of codimension $1$ and $B_1 \cap \breakpoints{g}$ has codimension $2$, the intersection $B_1 \cap \breakpoints{g}$ must be contained in the closure of $B_1 \setminus \breakpoints{g}$. Since the breakpoint set $\breakpoints{f+g}$ is closed, it follows that $B_1 \cap \breakpoints{g} \subseteq \breakpoints{f+g}$, proving the claim.
\end{proof}
As a consequence, bent hyperplanes from the final hidden layer always remain visible in the realized function.
\begin{lemma}
\label{lem:lastbenthyperplanesstay}
Let $\architecture=(n_0,n_1,\ldots,n_L,m)$ and let $\theta \in \paramspace{\architecture}$ be a supertransversal parameter. 
Then the bent hyperplanes from the last hidden layer are fully contained in the breakpoint complex, that is, $B_{L,j}(\theta) \subseteq \breakpoints{f_\theta}$ for each $j \in [n_L]$. 
\end{lemma}
\begin{proof}
It is easy to verify that for a vector-valued CPWL function $f \colon \R^d \to \R^m$, the overall breakpoint set is the union of the breakpoint sets of its coordinate functions: $\breakpoints{f} = \bigcup_{i \in [m]}\breakpoints{f_i}$.

For each $i \in [m]$, the $i$-th coordinate of the output is 
\[
(f_\theta)_i(x) = \bias{L+1}_i + \sum_{j \in [n_L]} \weights{L+1}_{i,j} \activation{L}_j(x).
\]
Hence, the breakpoints of $(f_\theta)_i$ depend on the breakpoints of the final activations $\activation{L}_j$. For every $j \in [n_L]$, the bent hyperplane $B_{L,j}(\theta)$ is exactly contained in $\breakpoints{\activation{L}_j}$ and is the support of a polyhedral complex pure of codimension $1$. Moreover, by supertransversality, the intersection of any two bent hyperplanes has codimension at least $2$. 

Applying \Cref{lem:unionofbreakpoints} iteratively over the sum yields that $B_{L,j}(\theta) \subseteq \breakpoints{(f_\theta)_i} \subseteq \breakpoints{f_\theta}$, proving the claim.
\end{proof}

With these tools in place, we are ready to prove our theorem on rigidity of visible hyperplanes. 

\thmintersectedhyperplanesstay*

\begin{proof}[Proof of \Cref{thm:intersected_hyperplanes_stay}] 
Let $H \in \hyperplanes{1}{\theta}$. By assumption, there exists $x \in H$ such that $\preactivation{2,\theta}_j(x) > 0$. We consider two cases based on the behavior of $\preactivation{2,\theta}_j$ on $H$:

\textbf{Case 1:} There also exists $x' \in H$ such that $\preactivation{2,\theta}_j(x') < 0$. 
By continuity, there must be a point on $H$ where $\preactivation{2,\theta}_j = 0$, meaning the second-layer bent hyperplane $B_{2,j}(\theta)$ intersects $H$. Because $\theta$ is supertransversal and honest, this intersection contains a bending ridge $\tau \in \complex_\theta$. 
By \Cref{lem:lastbenthyperplanesstay}, since $B_{2,j}(\theta)$ is from the last hidden layer ($L=2$), this bending ridge is preserved in the breakpoint complex, so $\tau \in \bcomplex{\theta}$. 

Because $f_\theta = f_\eta$ and both parameters are generic, \Cref{prop:canonicalbcomplex} implies $\bcomplex{\theta} = \bcomplex{\eta}$. Thus, $\tau$ is also a bending ridge in $\bcomplex{\eta}$. Applying \Cref{lem:bendingoptions} to $\bcomplex{\eta}$ allows us to uniquely identify which of the adjacent facets at $\tau$ originate from the earlier layer (layer $1$). Since hyperplanes in the first layer are global affine hyperplanes, 
tracking this facet uniquely identifies the entire hyperplane $H$. 
Therefore, $H \in \hyperplanes{1}{\eta}$. 

\textbf{Case 2:} $\preactivation{2,\theta}_j(y) \ge 0$ for all $y \in H$. 
Because $\theta$ is cancellation-free and $\preactivation{2,\theta}_j(x) > 0$ at some point on $H$, the first-layer hyperplane $H$ is not canceled; 
it actively contributes to a gradient change in $f_\theta$ and hence appears as global affine hyperplane in $\breakpoints{f_\theta} = \breakpoints{f_\eta}$. 

In the canonical polyhedral complex of a generic parameter, only the first hidden layer produces unbroken, globally affine hyperplanes (as generic second-layer bent hyperplanes must intersect hyperplanes from the first layer for $d,n_1>1$ and hence are bent). 
Therefore, to reproduce this flat hyperplane in $\bcomplex{\eta}$, it must originate from the first layer of $\eta$. 
Thus, $H \in \hyperplanes{1}{\eta}$. 

\medskip
Finally, if $H \in \hyperplanes{1}{\eta}$ holds for every $H \in \hyperplanes{1}{\theta}$, then we have $\hyperplanes{1}{\theta} \subseteq \hyperplanes{1}{\eta}$. 
Because $\theta$ and $\eta$ are generic, their first hidden layers have full rank and the same number of neurons $n_1$. 
Thus, the number of hyperplanes is exactly the same, which implies $\hyperplanes{1}{\theta} = \hyperplanes{1}{\eta}$ and therefore $\complex_{\theta,1}=\complex_{\eta,1}$. 
By \Cref{lem:lastbenthyperplanesstay}, $\complex_\theta$ is obtained by refining $\complex_{\theta,1}$ such that $\bcomplex{\theta}$ is a subcomplex. Since $\bcomplex{\eta}=\bcomplex{\theta}$, the entire canonical complexes must be identical: $\complex_\theta=\complex_\eta$. 
\end{proof}

\section{Layerwise Fibers}
\label{app:layerwise_fibers}
\addcontentsline{apx}{section}{\protect\numberline{\thesection}Layerwise Fibers}

\subsection{Hyperplane Representations}

We prove the lemma on the uniqueness of the hyperplane representation (\Cref{def:hyperplane-representation}). 

\lemhyperplanerepresentation* 

\begin{proof}
Let $f_\theta \colon \orthant{d} \to \R^m$ be a one-hidden-layer network.
Consider the hyperplane arrangement
$
\mathcal H=\{H_1,\dots,H_k\}
$
consisting of all hyperplanes induced by the nonlinear neurons of the first hidden layer that intersect the domain $\orthant{d}$. By construction, the breakpoint set of $f_\theta$ is contained in the support of this arrangement, although some of the hyperplanes may carry zero weight and hence may not contribute actual breakpoints of the realized function.

Since $f_\theta$ is CPWL, there exists a full-dimensional region
$
R \in \complex_{\mathcal H}\cap \orthant{d}
$
on which $f_\theta$ is affine linear, say
$
f_\theta(x)=A_Rx+b_R$ for all $ x\in R.
$

By Proposition~\ref{prop:tropical_weight_1layer}, the weight function $c_\theta$ is constant along each hyperplane $H_i$ of the arrangement. Hence for each $i\in[k]$ the quantity
$
c(i)\coloneqq c_\theta(\sigma)
$
is well defined for any facet $\sigma \subseteq H_i$.
If the total contribution of the neurons associated with $H_i$ cancels, then $c(i)=0$; in that case $H_i$ is present in the hyperplane representation but does not appear in the actual breakpoint set. Thus $(\mathcal H,A_R,b_R,c)$ is a hyperplane representation of $f_\theta$.

Now let $\eta$ be another parameter. Suppose that $(\mathcal H,A_R,b_R,c)$ is also a hyperplane representation of $f_\eta$. Then $f_\eta$ and $f_\theta$ are both compatible with the same arrangement $\mathcal H$, they agree with the same affine-linear map on the same reference region $R$, and they induce the same weight function $c$ on the facets of $\complex_{\mathcal H}$. By \Cref{lem:structure_theorem_tropical_geometry}, the two functions can differ only by a global affine-linear function, and since they agree on the full-dimensional region $R$, this difference must vanish identically. Hence
$f_\eta=f_\theta$. 

Conversely, suppose $f_\eta=f_\theta$. Then any hyperplane representation of $f_\theta$ is also a hyperplane representation of $f_\eta$: the same arrangement $\mathcal H$ still covers the breakpoint set, the same affine map $(A_R,b_R)$ is valid on the reference region $R$, and the induced weight function agrees because the realized functions coincide. Therefore $(\mathcal H,A_R,b_R,c)$ is a hyperplane representation of $f_\eta$ as well.
\end{proof}

\subsection{General Fibers} 
\label{app:general-fiber}

 If $(\{H_1,\ldots,H_k\}, A_R, b_R, c)$ is a hyperplane representation of $f_\theta \colon \orthant{d} \to \R^m$, 
 then the fiber $\fiber{\theta}$ consists of all parameters that admit this hyperplane representation. 
To describe this set, we first consider the different possible roles of the hidden neurons. 

\begin{definition}
For $\theta \in \paramspace{(d,n,m)}$, we call a neuron $i\in[n]$ \emph{dead} if
$\weights{1,\theta}_i x + \bias{1,\theta}_i \le 0$ for all $x\in\R_{\ge0}^d$ and  \emph{linear} if
$\weights{1,\theta}_i x + \bias{1,\theta}_i \ge 0$ for all $x\in\R_{\ge0}^d$, 
and it is not dead. A neuron is called
\emph{nonlinear} if it is neither dead nor linear. 
\end{definition} 

Given a hyperplane representation, 
each neuron is either nonlinear and associated with one of the hyperplanes, 
nonlinear and part of a canceling group of neurons, linear, or dead. 
Accordingly, we consider an assignment map $\phi \colon [n] \to \mathcal{L}$ with 
$\mathcal{L} = \{1, \dots, k\} \cup \{k+1, \dots, n\} \cup \{\text{lin}, \text{dead}\}$,  
which partitions the set of neurons into four subsets: 
\begin{itemize}[leftmargin=*]
    \item $N_{\text{vis}} = \phi^{-1}(\{1, \dots, k\})$: 
    neurons assigned to the visible hyperplanes $H_1,\ldots, H_k$.  
    \item $N_{\text{can}} = \phi^{-1}(\{k+1, \dots, n\})$: neurons assigned to \emph{canceling groups}. 
    For each $j\in\{k+1,\ldots, n\}$, 
    the neurons $\phi^{-1}(j)$ share a hyperplane in $\mathbb{R}_{\geq0}^d$ and collectively compute an affine function. 
    \item $N_{\text{lin}} = \phi^{-1}(\{\text{lin}\})$: neurons that are linear (always active) on $\orthant{d}$. 
    \item $N_{\text{dead}} = \phi^{-1}(\{\text{dead}\})$: neurons that are dead (always inactive) on $\orthant{d}$. 
\end{itemize}

\begin{definition}[Canonical fiber representation]
\label{def:semi-algebraic-set-V-phi-o}

For a given hyperplane representation $(\{H_1,\ldots,H_k\}, A_R, b_R, c)$, 
let 
$a_i \in \R^d$ and $t_i \in \R$ be the unique unit vectors and thresholds such that 
$H_i = \{ \inner{a_i,x}+ t_i=0\}$ and $R = \{x \in X \mid \inner{a_i, x} + t_i \le 0 \; \forall i \in [k]\}$. 
Then, for a fixed assignment $\phi$ and orientation $o \in \{1, -1\}^n$, we define the semi-algebraic set $V_{(\phi, o)}$ as the set of parameters $(\weights{1}, \bias{1}, \weights{2}, \bias{2})$ satisfying the following conditions:  
\begin{enumerate}[leftmargin=*]
    \item \emph{Normalization}: 
    For all neurons we fix the norm of the rows: $\|\weights{1}_i\|_2 = 1$ $\forall i \in [n]$. 
    
    \item \emph{Breakpoint Structure}: 
    Nonlinear neurons align with hyperplanes according to their assignment and orientation. 
    \begin{itemize}
        \item For visible neurons, the parameters are fixed as: 
        \begin{equation} \label{eq:f_geom_vis}
            (\weights{1}_i, \bias{1}_i) = o(i) (a_{\phi(i)}, t_{\phi(i)})  \;\; \forall i \in N_{\text{vis}}.
        \end{equation}
 
        \item For canceling neurons, any two in the same group, $j,i \in \phi^{-1}(\ell)$, share the same hyperplane and are oriented according to $o$ as: 
        \begin{equation} \label{eq:f_geom_can}
            (\weights{1}_i, \bias{1}_i) = o(i) o(j) (\weights{1}_{j}, \bias{1}_{j}) \;\; \forall i,j \in \phi^{-1}(\ell) , \;  \forall \ell \in 
            \{k+1,\ldots, n\}. 
        \end{equation}
    \end{itemize}

    \item \emph{Orthant Constraints}: 
    For linear and dead neurons: 
    \begin{equation} \label{eq:f_orthant}
        (\weights{1}_i, \bias{1}_i) \ge 0 \;\; \forall i \in N_{\text{lin}}, \quad (\weights{1}_i, \bias{1}_i) \le 0 \;\; \forall i \in N_{\text{dead}}.
    \end{equation}

    \item \emph{Weight Consistency}: 
    The output weights for each neuron group $j$ add to the corresponding weight:
    \begin{equation} \label{eq:f_jumps}
        \sum_{i \in \phi^{-1}(j)} o(i) \weights{2}_{:,i} = 
        \begin{cases} 
        c(j) & \text{if } j \in \{1, \dots, k\},  \\
        0 & \text{if } 
        j\in\{k+1,\ldots, n\}. 
        \end{cases}
    \end{equation}

    \item \emph{Affine Base Map}: 
    The neurons active on the reference region $R$ must match $(A_R, b_R)$. 
    A nonlinear neuron $i \in N_{\text{vis}} \cup N_{\text{can}}$ is active on $R$ if and only if $o(i) = -1$: 
    \begin{equation} 
        \sum_{i \in N_{\text{lin}}} \weights{2}_{:,i} \weights{1}_i + \sum_{\substack{i \in N_{\text{vis}} \cup N_{\text{can}} \\ o(i) = -1}} \weights{2}_{:,i} \weights{1}_i = A_R , 
    \label{eq:f_base_A}        
    \end{equation}
    \begin{equation} 
        \sum_{i \in N_{\text{lin}}} \weights{2}_{:,i} \bias{1}_i + \sum_{\substack{i \in N_{\text{vis}} \cup N_{\text{can}} \\ o(i) = -1}} \weights{2}_{:,i} \bias{1}_i + \bias{2} = b_R . 
    \label{eq:f_base_b}        
    \end{equation}
\end{enumerate}

We denote by $\fiber{\theta}\modulo{\sim}$ the fiber modulo positive rescaling and neuron permutation symmetries. 
\end{definition}

With these definitions in place, we can describe the fiber of a one-hidden-layer network as follows. 

\begin{restatable}{theorem}{thmonehiddenlayerfiber}
\label{thm:one-hidden-layer-fiber}
The fiber of a one-hidden-layer network $f_\theta \colon \R^d_{\ge 0} \to \R^m$ 
is given by  
$$
\fiber{\theta} \modulo{\sim} = \bigcup_{\phi, o} V_{(\phi, o)} \modulo{\sim},
$$ 
where the union is taken over all possible assignment maps $\phi$ and orientations $o$, and $V_{(\phi,o)}$ is the semi-algebraic set given in \Cref{def:semi-algebraic-set-V-phi-o} for the hyperplane representation of $f_\theta$.  
\end{restatable}

\begin{proof}[Proof of \Cref{thm:one-hidden-layer-fiber}] 
First we prove $\bigcup_{\phi, o} V_{(\phi, o)} \subseteq \fiber{\theta}$. 
Let $\eta \in V_{(\phi, o)}$. 
By \Cref{lem:hyperplane_representation}, it suffices to prove that $f_\eta$ has the same hyperplane representation
$(\{H_1,\ldots,H_k\}, A_R, b_R, c)$ as $\theta$. 
This follows by construction; we nonetheless verify the definition.
By the breakpoint structure, each nonlinear neuron of $\eta$ has a hyperplane that either aligns with one of the hyperplanes $H_1,\ldots,H_k$ or that cancels out with other nonlinear neurons. 
Specifically, if $\sigma \in \complex_\eta^{d-1}$ and $\sigma \subseteq H_j$, then, by \Cref{prop:tropical_weight_1layer}, for $j \in [k]$, we have that $$c_\eta(\sigma) =  \sum_{i \in \phi^{-1}(j)} [o(i) \weights{2}_{:,i}] = c(j). $$ For $j > k$, we have that $c_\eta(\sigma)=0$, so 
these hyperplanes do not appear in $\breakpoints{f_\eta}$.

The orthant constraints ensure $N_{\text{lin}}$ and $N_{\text{dead}}$ introduce no breakpoints in $\text{int}(\R_{\ge 0}^d)$, so $B({f_\eta})$ is covered by $\{H_1,\ldots,H_k\}$. 

On $R$, we have $\inner{a_j, x} + t_j \le 0$ for all $j\in[k]$ and hence a nonlinear neuron $i$ has preactivation $o(i)(\inner{a_{\phi(i)}, x} + t_{\phi(i)})$, which is $\ge 0$ (active) if and only if $o(i) = -1$. Linear neurons are always active. The last equation then ensures the affine map on $R$ matches $(A_R, b_R)$. Altogether, $(\{H_1,\ldots,H_k\}, A_R, b_R, c)$ is a hyperplane representation of $f_\eta$. 

Conversely, let $\eta =(\weights{1},\bias{1},\weights{2},\bias{2}) \in \fiber{\theta} \modulo{\sim}$. By the scaling symmetry, assume $\|\weights{1}_i\|_2 = 1$. 
We define $\phi$ by classifying neurons:
\begin{itemize}
    \item Neuron $i \in N_{\text{vis}}$ if $\{x \mid \inner{\weights{1}_i,x} +b_i=0\} =H_j \in \mathcal{H}$ with $c(j) \neq 0$.
    \item Neuron $i \in N_{\text{can}}$ if $H_j $ intersects $\text{int}(\R_{\ge 0}^d)$ but $c_\eta(\sigma) =0$ for $\sigma \subseteq \{x \mid \inner{\weights{1}_i,x} +b_i=0\}$. 
    \item We set $i \in N_{\text{lin}}$ or $i \in N_{\text{dead}}$ if the neuron is linear or dead, respectively. 
\end{itemize}

For $i \in N_{\text{vis}}$, the orientation $o(i)$   
is determined by the 
requirement $\weights{1}_i = o(i)a_{\phi(i)}$.  
For $i \in N_{\text{can}}$ one can choose either orientation such that opposite normals have opposite orientation. 
Because $f_\theta = f_\eta$, it follows that      $$\sum_{i \in \phi^{-1}(j)} o(i) \weights{2}_{:,i} = 
        \begin{cases} 
        c(j) & \text{if } j \in \{1, \dots, k\} \\
        0 & \text{if } 
        j\in\{k+1,\ldots, n\} 
        \end{cases}$$ 
        and the affine part on $R$ must match $(A_R, b_R)$. Thus $\theta$ satisfies the constraints of $V_{(\phi, o)}$. 
\end{proof}
\begin{remark}
    If the domain is $X = \R^d$ instead of $\mathbb{R}_{\geq0}^d$, then the situation simplifies as there are no inequality constraints. In this case, each component $V_{(\phi, o)}$ becomes an algebraic variety. 
\end{remark}

\subsection{Generic Fibers}  

For one-hidden-layer networks with generic parameters, the possible functional roles of hidden neurons simplify, and we have: 
\begin{enumerate}
    \item Every nonlinear neuron $i \in N_{\text{vis}}$ corresponds to a unique hyperplane $H_j \in \mathcal{H}$. Thus $k = |N_{\text{vis}}|$ and the assignment map $\phi$ restricted to $N_{\text{vis}}$ is a permutation.
    \item There are no canceling neurons. Thus $N_{\text{can}} = \emptyset$. 
\end{enumerate}

Next we prove our result describing the generic fiber of a one-hidden-layer network.

\propgenericlayerwisefiber* 

\begin{proof}[Proof of \Cref{prop:generic_layerwise_fiber}]
By minimality, there are no dead neurons, and by genericity, there are no breakpoint cancellations. Since there are exactly $k$ hyperplanes in $\mathcal{H}$, any generic parameter in the fiber must have exactly $k$ nonlinear neurons, one for each hyperplane, and $n-k$ linear neurons.

For each nonlinear neuron $i \in [k]$, there are exactly two normalized parameter configurations $(\weights{1}_i, \bias{1}_i)$ that generate the required hyperplane $H_i$. These two choices correspond to whether the neuron is active or inactive on the reference region $R$, which is determined by the orientation $o(i) \in \{1, -1\}$. Thus the subsets $S \subseteq [k]$ encode all possible orientation choices for the nonlinear neurons. Once such a subset $S$ is fixed, the output weights $\weights{2}_{:,i}$ of the nonlinear neurons are uniquely determined by the hyperplane weights $c(i)$.

After fixing the nonlinear neurons, the remaining affine map $(A_S,b_S)$ must be realized by the linear neurons indexed by $J=[n]\setminus[k]$. The semi-algebraic set $V_S$ consists exactly of the possible parameters of these linear neurons and the output bias that realize this affine map subject to the positivity constraints ensuring that the neurons remain linear on the orthant. Hence each set
$
K_S=\{p_S\}\times V_S
$
is a semi-algebraic subset of the full fiber $\fiber{\theta}$.

Conversely, every generic parameter in the fiber arises in this way: its nonlinear neurons determine a unique subset $S\subseteq[k]$, and its linear neurons must realize the corresponding affine compensation $(A_S,b_S)$. Therefore every generic parameter in the fiber lies in exactly one of the sets $K_S$.

It follows that
$
\gfiber{\theta}\modulo{\sim}
=
\bigcup_{S\subseteq[k]} \left(K_S\cap \gparamspace{\architecture}\right) \modulo{\sim},
$
as claimed.
\end{proof}

\subsection{Dimension of Generic Fibers} 
\label{subsec:dim_generic_fibres}

The dimensions and non-emptiness of the fiber components $K_S$ in \Cref{prop:generic_layerwise_fiber} are fundamentally tied to the ability to decompose the required linear part $A_S$ as a sum of linear neurons. 
This is captured by the notion of nonnegative column rank. 

\begin{definition}
Let $A \in \R^{m \times d}$ be a real matrix. The \emph{nonnegative column rank} of $A$, denoted $\nncr{A}$, is the smallest integer $k$ such that there exist matrices $U \in \R^{m \times k}$ and $V \in \R_{\ge 0}^{k \times d}$ satisfying $A = UV$.
\end{definition}

\begin{lemma}
\label{lem:nncratmostL}
The semi-algebraic set $K_S$ is non-empty if and only if the nonnegative column rank of $A_S$ is at most $n-k$. 
\end{lemma}

\begin{proof}
By definition, if $K_S\neq\emptyset$, then there exist parameters in $V_S$ such that
\(
\weights{2}_{:,J}\weights{1}_J=A_S
\)
with $\weights{1}_J\ge0$ and $\|\weights{1}_i\|_2=1$ for all $i\in J$. Any such point gives a nonnegative factorization $A_S=UV$ with $U=\weights{2}_{:,J}\in\R^{m\times|J|}$ and $V=\weights{1}_J\in\R_{\ge0}^{|J|\times d}$, and hence $\nncr{A_S}\le|J|$. 

Conversely, if $\nncr{A_S}\le|J|$, take a nonnegative factorization $A_S=BD$. 
Normalizing the rows of $D$ and compensating by scaling the columns of $B$ preserves the product and satisfies the constraints of $V_S$, 
while $\bias{2}$ can be chosen to meet the affine condition, so $K_S\neq\emptyset$. 
\end{proof}

\begin{lemma}
\label{lem:minimality_bounds_linear_neurons}
Let $\theta \in \paramspace{(d,n,m)}$ be generic and minimal, and let $k$ denote the number
of nonlinear neurons. Then
$
n-k \le \min\{d,m+1\}.
$
\end{lemma}

\begin{proof}
Let $J=[n]\setminus[k]$. Since $\theta$ is generic and minimal, there are no dead neurons,
so the neurons in $J$ are exactly the linear neurons.

Their contribution is an affine map
$
x \mapsto \weights{2}_{:,J}\weights{1}_J x + \weights{2}_{:,J}\bias{1}_J + \bias{2}.
$
Write
$
A:=\weights{2}_{:,J}\weights{1}_J.
$
Because $\weights{1}_J\ge 0$, this is a nonnegative factorization of $A$ through $|J|=n-k$
hidden units.

By \Cref{cor:nncr=rank}, if $\rank(A)=d$, then $\nncr{A}=d$, and otherwise
$\nncr{A}\le \rank(A)+1\le d$ by \Cref{prop:nnrc<rank+1}. Hence $\nncr{A}\le d$.
Also \Cref{prop:nnrc<rank+1} gives $\nncr{A}\le m+1$. Therefore
$
\nncr{A}\le \min\{d,m+1\}.
$

If $n-k>\min\{d,m+1\}$, then the same affine map can be realized by strictly fewer than
$n-k$ linear neurons. Replacing the linear neurons in $\theta$ by such a smaller realization,
while keeping the $k$ nonlinear neurons fixed, yields a strict subarchitecture realizing the
same function, contradicting minimality. Thus
$n-k\le \min\{d,m+1\}$. 
\end{proof}
\begin{prop}
\label{prop:empty_or_dim}
Let $\theta \in \paramspace{(d,n,m)}$ be generic and minimal and let $k$ denote the number of nonlinear neurons. If with $n-k \neq m+1$, then, for every $S \subseteq [k]$, either $\dim(K_S) = (n-k)^2$ or $K_S=\emptyset$. 
\end{prop}

\begin{proof} 
By \Cref{lem:minimality_bounds_linear_neurons}, we have
$n-k \le \min\{d,m+1\}$. 
Consider a fixed subset $S \subseteq [k]$ and let $(A_S, b_S)$ denote the affine part of $f_\theta$ after flipping the orientation of neurons in $S$. 
Define 
\[
D_S := \{(U,V,u,v) \in \R^{m \times (n-k)} \times \R^{(n-k) \times d} \times \R^m \times \R^{(n-k)} \mid UV = A_S, \; Uv+u = b_S, \; V,v \ge 0\},
\]
and denote the projection of $D_S$ onto the matrix components as 
\[
M_S := \{(U,V) \in \R^{m \times (n-k)} \times \R^{(n-k) \times d} \mid U V = A_S, \; V \ge 0\} . 
\]
If $M_S=\emptyset$, then also $D_S = \emptyset$, and $K_S = \emptyset$. 

Suppose $M_S\neq\emptyset$ and fix a feasible point $(U_0,V_0) \in M_S$. 
For any $$X \in \GL_{n-k}^+ \coloneqq \{ X \in \GL_{n-k} \mid X \geq 0\},$$ define $(U,V) = (U_0 X^{-1}, X V_0)$. Then $$U V = U_0 X^{-1} X V_0 = U_0 V_0 = A_S,$$ so the pair $(U,V)$ satisfies the factorization constraint and hence is contained in $M_S$. By genericity, $V_0$ has full rank $n-k$, which implies that the map $X \mapsto (U_0 X^{-1}, X V_0)$ is injective. Since $\GL_{n-k}^+$ contains an open set in $\GL_{n-k}$, this shows that 
$\dim(M_S) \ge (n-k)^2$. 
Equality follows from the fact that $$M_S \subseteq \overline{M_S} = \{(U,V) \in \R^{m \times n-k} \times \R^{(n-k) \times d} \mid U V = A_S\},$$ which is parameterized by $\GL_{n-k}$ and hence has dimension $(n-k)^2$.

Next, consider the linear constraint $U v + u = b_S$. For any choice of $v \in \R_{\geq 0}^{(n-k)}$, there exists a unique $u \in \R^m$ satisfying this equation, namely $u = b_S - Uv$. 
Therefore, the translation degrees of freedom contribute exactly $n-k$ dimensions, corresponding to the entries of $v$. 
Passing from $D_S$ to $K_S$ imposes that each row of $V$ has unit Euclidean norm. This removes exactly $n-k$ degrees of freedom from the GL orbit, while leaving the translation degrees of freedom parametrized by $v$ is unaffected. 

Combining these contributions, the dimension of $K_S$ is
$$
\dim(K_S) = (n-k)^2 + (n-k) - (n-k) = (n-k)^2.
$$
This completes the proof. 
\end{proof}

\begin{prop}
\label{prop:nnrc<rank+1}
For any matrix \(A \in \mathbb{R}^{m \times d}\), 
the nonnegative column rank satisfies \(\nncr{A} \le \mathrm{rank}(A) + 1\). 
\end{prop}

\begin{proof} 
Let \(r = \mathrm{rank}(A)\), and let \(C \in \mathbb{R}^{m \times r}\) be a matrix whose columns form a basis of the column space of \(A\), so that \(A = C F\) for some \(F \in \mathbb{R}^{r \times d}\). 

In \(\mathbb{R}^r\), consider a minimal complete simplicial fan with \(r+1\) rays \(v_1, \dots, v_{r+1} \in \mathbb{R}^r\). Then each column \(F_j\) of \(F\) can be written as a nonnegative combination of these rays: 
\(F_j = \sum_{i=1}^{r+1} \lambda_{ij} v_i\) with \(\lambda_{ij} \ge 0\). 
If we let \(V \in \mathbb{R}^{r \times (r+1)}\) be the matrix with columns \(v_1, \dots, v_{r+1}\) and define \(D := (\lambda_{ij}) \in \mathbb{R}_{\ge 0}^{(r+1) \times d}\), then \(F = V D\). 

Thus, setting \(B := C V \in \mathbb{R}^{m \times (r+1)}\), we have \(A = B D\). 
This is a factorization with nonnegative \(D\) and inner dimension \(r+1\). 
Therefore, \(\mathrm{rank}_{\mathrm{col},+}(A) \le r + 1 = \mathrm{rank}(A) + 1\), as claimed. 
\end{proof}

\begin{prop}
\label{prop:whennncr=rank}
Let $A \in \mathbb{R}^{m \times d}$ have rank $r \ge 1$. 
Then $\nncr{A} = r$ 
if and only if the columns of $A$ are contained in a single simplicial cone in $\mathrm{col}(A)$. 
\end{prop}

\begin{proof}
Let $C \in \mathbb{R}^{m \times r}$ be a matrix whose columns form a basis of $\mathrm{col}(A)$, so that there exists $F \in \mathbb{R}^{r \times d}$ with $A = C F$. 

First, note that any $r$ linearly independent vectors in $\mathbb{R}^r$ generate a simplicial cone. Indeed, if $v_1, \dots, v_r \in \mathbb{R}^r$ are linearly independent, then 
\[
\mathrm{cone}(v_1, \dots, v_r) := \Big\{\sum_{i=1}^r \lambda_i v_i : \lambda_i \ge 0 \Big\}
\]
is a simplicial cone in $\mathbb{R}^r$. 

If all columns of $F$ lie in a single such cone, then every column of $F$ is a nonnegative combination of these $r$ vectors, and hence $\mathrm{rank}_{\mathrm{col},+}(A) \le r$. 

Conversely, if the columns of $F$ are not contained in a common simplicial cone, then at least $r+1$ vectors are required to generate a cone containing all columns. 
Therefore, any factorization $A = B D$ with $D \ge 0$ must have at least $r+1$ columns in 
$B$, 
and we conclude that $\mathrm{rank}_{\mathrm{col},+}(A) = r+1$. 
\end{proof}

\begin{cor}
\label{cor:nncr=rank}
Let $A \in \mathbb{R}^{m \times d}$ with $d \leq m$. Then $\nncr{A}\leq d$.
\end{cor}

\begin{proof}
Let $r = \mathrm{rank}(A) \le d$. By \Cref{prop:nnrc<rank+1}, we know that
\(
\nncr{A} \le \mathrm{rank}(A) + 1 = r + 1 \le d + 1.
\)
If $r < d$, then $r+1 \le d$, and $\nncr{A} \le d$. 
If $r = d$, then $A$ has full column rank. In this case, the columns of $A$ are $d$ linearly independent vectors in $\mathbb{R}^m$ 
and generate a simplicial cone. 
Hence, by \Cref{prop:whennncr=rank}, $\nncr{A} = r = d$.  
Combining these two cases shows
\(
\nncr{A} \le d
\), as claimed. 
\end{proof} 
\begin{prop}
\label{prop:dim_special_case}
Let $\theta \in \paramspace{(d,n,m)}$ be generic and minimal and let $k$ denote the number of nonlinear neurons. If $n-k = m+1$, then, for every $S \subseteq [k]$, we have
$
\dim(K_S)= m^2 +m +d. 
$
\end{prop}

\begin{proof}
Let $M_S$ be as in the proof of \Cref{prop:empty_or_dim}. Since $\theta$ is minimal, we have $d \geq m+1$. By genericity, $\rank(A_S) = m$. 

We first compute the dimension of
$
\overline{M_S}
=
\{(U,V)\in \R^{m\times (m+1)}\times \R^{(m+1)\times d}\mid UV=A_S\}.
$
Fix $U \in \R^{m \times (m+1)}$ with $\rank(U)=m$. Then $\dim\ker(U)=1$, so for each column $a_j$ of $A_S$, the solution set of $Ux=a_j$ is an affine line. Hence
$
\{V \in \R^{(m+1) \times d} \mid UV =A_S\}
$
is an affine space of dimension $d$. Therefore
$
\dim(\overline{M_S}) = m(m+1)+d= m^2 + m +d.
$

Now let
$
Z= \{ U \in \R^{m \times (m+1)} \mid \exists v \in \R^{m+1}_{>0}:  \ker(U) = \operatorname{span}(v)\}.
$
This is an open subset of $\R^{m\times (m+1)}$, hence
$
\dim(Z)=m(m+1)=m^2+m.
$

For $U \in Z$, define
$
M_S(U) \coloneqq \{V \in \R^{(m+1) \times d} \mid UV =A_S ,\ V_{i,j} \geq 0\}.
$
We claim that $\dim(M_S(U))=d$. Indeed, since $\rank(U)=m$, the affine space
$
\{V \in \R^{(m+1) \times d} \mid UV =A_S\}
$
has dimension $d$. Choose $c>0$ sufficiently large so that
$
Y_c = \{V \in \R^{(m+1) \times d} \mid UV =A_S ,\ |V_{i,j}| \leq c\}
$
has nonempty interior relative to this affine space; in particular, $\dim(Y_c)=d$.

Let $v\in \R^{m+1}_{> 0}$ satisfy $\ker(U) = \operatorname{span}(v)$ and set
$
\delta = \min_i v_i >0.
$
Define
$
\varphi_{c,v}\colon
\begin{bmatrix}
V_1 & \cdots & V_d
\end{bmatrix}
\mapsto
\begin{bmatrix}
V_1 + \frac{c}{\delta}v & \cdots & V_d +\frac{c}{\delta}v
\end{bmatrix}.
$
Since $Uv=0$, the map $\varphi_{c,v}$ preserves the equation $UV=A_S$. Moreover, for every entry we have
$
(V_j)_i + \frac{c}{\delta}v_i \ge -c + \frac{c}{\delta}\delta =0,
$
so $\varphi_{c,v}(Y_c) \subseteq M_S(U)$. Since $\varphi_{c,v}$ is an affine isomorphism, it follows that
$
\dim(M_S(U)) \ge d.
$
On the other hand,
$
M_S(U)\subseteq \{V\in \R^{(m+1)\times d}\mid UV=A_S\},
$
and the latter has dimension $d$, hence $\dim(M_S(U))=d$.

Therefore
$
M_S\cap (Z\times \R^{(m+1)\times d})
$
projects onto $Z$ with all fibers of dimension $d$. By the fiber dimension theorem for semi-algebraic sets,
$
\dim(M_S)\ge m^2+m+d.
$
Since $M_S\subseteq \overline{M_S}$ and $\dim(\overline{M_S})=m^2+m+d$, we conclude that
$
\dim(M_S)=m^2+m+d.
$

Finally, passing from $M_S$ to $K_S$ is analogous to the proof of \Cref{prop:empty_or_dim}: the translation variables contribute $m+1$ degrees of freedom, while the row-normalization constraints remove exactly $m+1$ degrees of freedom, and the output bias is then uniquely determined. Hence 
$\dim(K_S)=\dim(M_S)=m^2+m+d$. 
\end{proof}

With these tools in place, we are ready to prove our theorem on the dimension of the generic fibers. 

\thmlayerwisefibers* 

\begin{proof}[Proof of \Cref{thm:layerwise-fibers}]

Let $J=[n]\setminus[k]$.
For all $i \in [k]$ we have that 
\begin{equation}
\label{eq:linearpart_flipped_neurons}
     c(i) a_i^\top =  o_\theta(i)\weights{2,\theta}_{:,i} \weights{1,\theta}_i . 
\end{equation}
Recall that on the reference region $R$ we have that $$f_\eta(x)=A_Rx+b_R$$ for all $\eta\in\mathcal{S}(\theta).$ Further, we have that $A_R,b_R$ split into contributions from the linear and non-linear neurons. For $\theta$, let $o_\theta(i)$ be the hyperplane orientation and $S_\theta$ the set of nonlinear active neurons on $R.$ Then $A_R=A_{S_\theta}-\sum_{S_\theta}c(i)a_i^\top.$
Now choose a different orientation $o\in\{1,-1\}^k$ of the hyperplanes, and let $S = \{i\in[k]\mid o(i)=-1\}$. This needs to be compensated by the contribution of the linear neurons as follows:
\begin{eqnarray*}
    A_R&=&A_{S_\theta}-\sum_{S_\theta}c(i)a_i^\top\\
    &=&A_{S_\theta}-\sum_{S_\theta\setminus S}c(i)a_i^\top-\sum_{S\bigcap S_\theta}c(i)a_i^\top-\sum_{S\setminus S_\theta}c(i)a_i^\top + \sum_{S\setminus S_\theta}c(i)a_i^\top\\
    &=&A_{S_\theta}-\sum_{S_\theta\setminus S}c(i)a_i^\top+\sum_{S\setminus S_\theta}c(i)a_i^\top-\sum_{S}c(i)a_i^\top. 
\end{eqnarray*}
Thus, for this new set $S,$ the contribution from the linear neurons is required to be 
\[
A_S = A_R+ \sum_S c(i)a_i^\top =A_{S_\theta}-\sum_{S_\theta\setminus S}c(i)a_i^\top+\sum_{S\setminus S_\theta}c(i)a_i^\top . 
\]

Using \Cref{eq:linearpart_flipped_neurons}, we obtain \begin{align*}
   A_S&= \weights{2}\diag(\ind_J)\weights{1} -  \sum_{i \in S_\theta \setminus S}  o_\theta(i)\weights{2,\theta}_{:,i} \weights{1,\theta}_i + \sum_{i \in S \setminus S_\theta}  o_\theta(i)\weights{2,\theta}_{:,i} \weights{1,\theta}_i \\&= \weights{2}\diag(\ind_J)\weights{1} +  \sum_{i \in S_\theta \setminus S}  \weights{2,\theta}_{:,i} \weights{1,\theta}_i + \sum_{i \in S \setminus S_\theta}  \weights{2,\theta}_{:,i} \weights{1,\theta}_i \\&= \weights{2}(\diag(\ind_{J \cup (S \triangle S_\theta)}) \weights{1} . 
\end{align*}
Essentially, this means we have to add the part from neurons that are in $S_\theta$ but not in $S$ and subtract the part from neurons that are in $S$ but not in $S_\theta$ from the old linear part.
Similar computation for the constant part. 
This allows us to describe the fiber as follows: we run through all choices of orienting the hyperplanes, each inducing a $A_S$ and then find all factorizations such that the linear neurons multiply to $A_S.$

By genericity (\Cref{def:generic_parameter}),
\[
\rank(A_S)=\min\{d,m, |J\cup (S_\theta \triangle S)|\}. 
\]

\medskip
\noindent
(1) Assume $n-k<\min\{d,m\}$. 
If $S\neq S_\theta$, then $|J\cup(S_\theta \triangle S)|>n-k$ and $d,m> n-k$, 
hence 
$\rank(A_S)>n-k$. 
Since $\nncr{A_S}\ge \rank(A_S)$, 
\Cref{lem:nncratmostL} implies $V_S=\emptyset$. 
If $S=S_\theta$, then 
$\nncr{A_S}=n-k$ and $V_S\neq\emptyset$. 
The dimension statement follows from 
\Cref{prop:empty_or_dim}.

\medskip
\noindent
(2) Assume $n-k=d\le m$. Then $\rank(A_S)=n-k$ for all $S$. 
By \Cref{cor:nncr=rank}, 
$\nncr{A_S}=\rank(A_S)=n-k$. 
Hence $V_S\neq\emptyset$ for all $S$, and 
\Cref{prop:empty_or_dim} yields 
$\dim(V_S)=(n-k)^2$. 

\medskip
\noindent
(3) 
Assume $n-k = m<d$. 
By genericity (\Cref{def:generic_parameter}) we have $\rank(A_S) = m$. 
By \Cref{lem:nncratmostL},
$V_S \neq \emptyset$ if and only if $\nncr{A_S} \le n-k = m$.  
Let $C = \cone(A_S) \subseteq \col(A_S)$ denote the conic hull of the columns of $A_S$. By \Cref{prop:whennncr=rank}, $\nncr{A_S} = m$ if and only if the columns of $A_S$ are contained in a common simplicial cone in $\col(A_S)$. 

We show that $C$ is contained in a simplicial cone if and only if it is pointed.  
 Let $C^\circ = \{y \in \mathbb{R}^m : \langle y, x \rangle \ge 0 \text{ for all } x \in C\}$ denote the polar cone of $C$. 
 The cone $C$ is pointed if and only if $C^\circ$ is full-dimensional. 
 \begin{itemize}[leftmargin=*]
     \item 
 If $C$ is pointed, choose $m$ linearly independent vectors $v_1,\dots,v_m \in \operatorname{int}(C^\circ)$. Let $V := \cone(v_1,\dots,v_m) \subseteq C^\circ$. Then $V$ is a simplicial, full-dimensional cone contained in $C^\circ$. 
Take the polar $V^\circ$ of $V$. This is a simplicial cone and, by polarity and inclusion reversal, we have $C \subseteq V^\circ$.  
\item 
Conversely, if $C$ is not pointed, it contains a line and cannot be contained in any simplicial cone.  
 \end{itemize}
Combining this with \Cref{lem:nncratmostL} and \Cref{prop:empty_or_dim}, we conclude
\[
\dim(V_S) = (n-k)^2 \quad \Longleftrightarrow \quad \cone(A_S) \text{ is pointed}, 
\]  
and $V_S = \emptyset$ otherwise.

\medskip
\noindent
(4) Follows by \Cref{prop:dim_special_case}.
\end{proof}

\section{Fibers from Layer Composition} 
\addcontentsline{apx}{section}{\protect\numberline{\thesection}Fibers from Layer Composition} 

\subsubsection*{The Continuous Symmetry}

\lemtranslatinghiddenhyperplaneinclusion*
\begin{proof}[Proof of \Cref{lem:translating_hidden_hyperplane_inclusion}]
Let $\eta \in T_j(\theta)$. Then there exists $t > \max_{i \in [n_2]} \frac{\bias{2}_i}{\weights{2}_{ij}}$ 
such that 
\(
\bias{1,\eta} = \bias{1} + t e_j\) 
and 
\(
\bias{2,\eta} = \bias{2} - t (\weights{2}_{1 j}, \dots, \weights{2}_{n_2 j})^\top
\), 
while all other weights and biases coincide with those of $\theta$. 

Consider the second-layer preactivation for an arbitrary input $x \in \R^d$: 
\[
\preactivation{2,\eta}_i(x) = \sum_{k=1}^{n_1} \weights{2}_{ik} \activation{1,\eta}_k(x) + 
\bias{2,\eta}_i
=
\left(\sum_{k \neq j} \weights{2}_{ik} \activation{1,\theta}_k(x)\right) + \weights{2}_{ij}\left[\weights{1}_{j}x + \bias{1}_j + t\right]_+ + \bias{2,\eta}_i .
\]

We now prove that $\activation{2,\eta}_i(x) = \activation{2,\theta}_i(x)$ by showing that the preactivations $\preactivation{2,\eta}_i$ and $\preactivation{2,\theta}_i$ coincide whenever they are positive. 
Set $y = \weights{1}_{j}x + \bias{1}_j$. We use the identity
\[
\preactivation{2,\eta}_i(x)-\preactivation{2,\theta}_i(x)
=
\weights{2}_{ij}\big([y+t]_+-[y]_+-t\big).
\]

 We proceed by 
 case distinction over the activation patterns of the first layer.  

 \begin{enumerate} 

\item 
If $y>0$ and $y+t>0$, then
$\preactivation{2,\eta}_i(x) - \preactivation{2,\theta}_i(x)=0$. 

\item 
If $y>0$ and $y+t\leq0$, then $t<0$ and $\preactivation{2,\eta}_i(x) \geq \preactivation{2,\theta}_i(x)$. 
Moreover, it holds that 
\[
\preactivation{2,\eta}_i(x) \leq \weights{2}_{ij}[y + t]_+ + \bias{2,\eta}_i= \bias{2,\eta}_i 
= 
\bias{2}_i-t
{\weights{2}_{ij}} 
\leq 0 .
\] 
For the first inequality we used 
$\weights{2}_{ik}\leq0$ for all $k\neq j$; 
for the second equality we used $y+t\leq0$; 
and for the last inequality we used $\bias{2}_i \leq 0$, $\weights{2}_{ij}>0$, and $t\leq0$. 

\item 
If $y\leq0$ and $y+t>0$, then $t>0$ and 
$\preactivation{2,\eta}_i(x) \leq \preactivation{2,\theta}_i(x)$. Moreover, it holds that \[\preactivation{2,\theta}_i(x) \leq \weights{2}_{ij}[y]_+ + \bias{2}_i=  \bias{2}_i\leq 0.\] 

\item 
If $y\leq0$ and $y+t\leq0$, then we have that  \(\preactivation{2,\theta}_i(x) \leq 0 \) as well as \(\preactivation{2,\eta}_i(x) \leq 0 \). 
\end{enumerate} 

In all four cases, $\activation{2,\eta}_i(x) = \activation{2,\theta}_i(x)$, and thus $\eta \in \gfiber{\theta}$. 
Since $\eta$ was arbitrary in $T_j(\theta)$, it follows that $T_j(\theta)\subseteq\gfiber{\theta}$, as claimed. 
\end{proof}

\subsubsection*{The Discrete Symmetry} 
The group $\mathcal{G}_j$ is generated by the singleton sets $\{k\}$ for $k \in \{1,\dots,n_1\}\setminus\{j\}$, i.e.,
\(
\mathcal{G}_j = \big\langle M_{\{k\}} : k \in \{1,\dots,n_1\} \setminus \{j\} \big\rangle.
\)

\propflippinghiddenhyperplaneinclusion* 

\begin{proof}[Proof of \Cref{prop:flipping_hidden_hyperplane_inclusion}] 
It suffices to prove the claim for $I=\{i\}.$
Let $i \in [n_1]\setminus\{j\}$ and define
\[
\eta \coloneqq M_{\{i\}} \cdot \theta = (M_{\{i\}}\weights{1}, M_{\{i\}} \bias{1}, \weights{2}, \bias{2}, \weights{3}, \bias{3}).
\]
The first-layer preactivations under 
$\eta$ are 
\[
\preactivation{1,\eta}_k(x) =
\begin{cases}
\preactivation{1,\theta}_j(x) - \preactivation{1,\theta}_i(x), & k = j,\\[1mm]
-\preactivation{1,\theta}_k(x), & k = i,\\[1mm]
\preactivation{1,\theta}_k(x), & k \notin \{i,j\}.
\end{cases}
\]
The second-layer preactivation is
\(
\preactivation{2,\eta}_1(x) = \sum_{k=1}^{n_1} \weights{2}_{1k} \activation{1,\eta}_k(x) + \bias{2}_1.
\)

We now prove that $\activation{2,\eta}_1(x) = \activation{2,\theta}_1(x)$. 
Fix a region $R \in \complex_{1,\theta}$ and let
$$
S = \{ k \in [n_1] \mid \activation{1,\theta}_k(x) > 0 \text{ for } x \in R \}
$$
denote the the set of first-layer neurons active on $R$ under $\theta$. 
Let $x \in R$. We distinguish four cases:

\medskip
\noindent \textbf{Case 1: $i \notin S,\, j \in S$.}  

Then $\preactivation{1,\theta}_i(x) \leq 0$ and $\preactivation{1,\theta}_j(x) > 0$. Hence
\[
\activation{1,\eta}_i(x) = \left[-\preactivation{1,\theta}_i(x)\right]_+ = -\preactivation{1,\theta}_i(x),
\]
and
\[
\activation{1,\eta}_j(x) = \left[\preactivation{1,\theta}_j(x) - \preactivation{1,\theta}_i(x)\right]_+
= \preactivation{1,\theta}_j(x) - \preactivation{1,\theta}_i(x).
\]
All other activations remain unchanged. Therefore
\[
\begin{aligned}
\preactivation{2,\eta}_1(x)
&=
\weights{2}_{1j}\big(\preactivation{1,\theta}_j(x) - \preactivation{1,\theta}_i(x)\big)
+ \weights{2}_{1i}\big(-\preactivation{1,\theta}_i(x)\big)
+ \sum_{k \notin \{i,j\}} \weights{2}_{1k} \activation{1,\theta}_k(x)
+ \bias{2}_1 \\
&=
\weights{2}_{1j}
\activation{1,\theta}_j(x)
+ \weights{2}_{1i}\activation{1,\theta}_i(x)
+ \sum_{k \notin \{i,j\}} \weights{2}_{1k} \activation{1,\theta}_k(x)
+ \bias{2}_1 \\
&= \preactivation{2,\theta}_1(x),
\end{aligned}
\]
where we used that 
the second-hidden-layer neuron is hiding hyperplane $j$ and is normalized, so that $\weights{2}_{1j} =+1$ and $\weights{2}_{1i}= -1$, and that 
$\activation{1,\theta}_i(x)=0$. 
Hence
\(
\activation{2,\eta}_1(x) = \activation{2,\theta}_1(x).
\)

\medskip
\noindent \textbf{Case 2: $i,j \in S$.}  

Then $\preactivation{1,\theta}_i(x) > 0$ and $\preactivation{1,\theta}_j(x) > 0$. Thus
\[
\activation{1,\eta}_i(x) = 0,
\qquad
\activation{1,\eta}_j(x) = \left[\preactivation{1,\theta}_j(x) - \preactivation{1,\theta}_i(x)\right]_+.
\]
We compare the second-layer preactivations:
\[
\preactivation{2,\theta}_1(x)
=
\weights{2}_{1j}\preactivation{1,\theta}_j(x)
+ \weights{2}_{1i}\preactivation{1,\theta}_i(x)
+ \sum_{k \notin \{i,j\}} \weights{2}_{1k} \activation{1,\theta}_k(x)
+ \bias{2}_1,
\]
\[
\preactivation{2,\eta}_1(x)
=
\weights{2}_{1j}\left[\preactivation{1,\theta}_j(x) - \preactivation{1,\theta}_i(x)\right]_+
+ \sum_{k \notin \{i,j\}} \weights{2}_{1k} \activation{1,\theta}_k(x)
+ \bias{2}_1.
\]

If $\preactivation{1,\theta}_j(x) \le \preactivation{1,\theta}_i(x)$, then
\(
\activation{1,\eta}_j(x)=0
\)
and hence
\(
\preactivation{2,\eta}_1(x) \le 0.
\)
Since $\weights{2}_{1i} < 0$, also $\preactivation{2,\theta}_1(x) \le 0$, and thus
\[
\activation{2,\eta}_1(x) = \activation{2,\theta}_1(x) = 0.
\]

If $\preactivation{1,\theta}_j(x) > \preactivation{1,\theta}_i(x)$, then
\[
\preactivation{2,\eta}_1(x)
=
\weights{2}_{1j}\big(\preactivation{1,\theta}_j(x) - \preactivation{1,\theta}_i(x)\big)
+ \sum_{k \notin \{i,j\}} \weights{2}_{1k} \activation{1,\theta}_k(x)
+ \bias{2}_1,
\]
which is positive if and only if $\preactivation{2,\theta}_1(x)$ is positive. In this case, both expressions coincide, and hence
\[
\activation{2,\eta}_1(x) = \activation{2,\theta}_1(x).
\]

\medskip
\noindent \textbf{Case 3: $j \notin S,\, i \notin S$.}  

Then $\preactivation{1,\theta}_j(x) \le 0$ and $\preactivation{1,\theta}_i(x) \le 0$ for all $x \in R$ and $\preactivation{1,\eta}_j(x)
=
\preactivation{1,\theta}_j(x) - \preactivation{1,\theta}_i(x). $ Moreover,
\[
\activation{1,\eta}_i(x) = \left[-\preactivation{1,\theta}_i(x)\right]_+ = -\preactivation{1,\theta}_i(x) \ge 0.
\]
All other neurons are unchanged. 
If $\preactivation{1,\eta}_j(x) \leq 0$, then since $\weights{2}_{1k} < 0$ for all $k \neq j$ and $\bias{2}_1 < 0$, we obtain
\[
\preactivation{2,\eta}_1(x) \le 0
\quad \text{and} \quad
\preactivation{2,\theta}_1(x) \le 0.
\]
Thus
\[
\activation{2,\eta}_1(x) = \activation{2,\theta}_1(x) = 0.
\]
On the other hand, if $\preactivation{1,\eta}_j(x) > 0$, then since $\weights{2}_{1k} \in \{1,-1\}$, we have that \[\sum_{k =1}^{n_2} \weights{2}_{1k} \activation{1,\eta}_k = \sum_{k \neq i, j} -\activation{1,\eta}_k + \preactivation{1,\theta}_j(x) - \preactivation{1,\theta}_i(x) + \preactivation{1,\theta}_i(x) < 0.\] Since $\bias{2,\eta} < 0$, this implies that \[
\activation{2,\eta}_1(x) = \activation{2,\theta}_1(x) = 0.
\]

\medskip
\noindent \textbf{Case 4: $j \notin S,\, i \in S$.}  

Then $\preactivation{1,\theta}_j(x) \le 0$ and $\preactivation{1,\theta}_i(x) > 0$. Hence
\[
\preactivation{1,\eta}_j(x)
=
\preactivation{1,\theta}_j(x) - \preactivation{1,\theta}_i(x)
< 0,
\]
so $\activation{1,\eta}_j(x)=0$, and
\[
\activation{1,\eta}_i(x) = \left[-\preactivation{1,\theta}_i(x)\right]_+ = 0.
\]
All other neurons are unchanged. 
Again, since all contributing weights except possibly $j$ are negative and $\bias{2}_1 < 0$, we obtain
\[
\preactivation{2,\eta}_1(x) \le 0
\quad \text{and} \quad
\preactivation{2,\theta}_1(x) \le 0.
\]
Thus
\[
\activation{2,\eta}_1(x) = \activation{2,\theta}_1(x) = 0.
\]

\medskip
\noindent
In all cases, we conclude
\[
\activation{2,\eta}_1(x) = \activation{2,\theta}_1(x) \quad \forall x \in R.
\]
Since $R$ was arbitrary, this holds for all $x \in \R^d$, and thus $\eta \in \gfiber{\theta}$. Since for $I=\{i_1,\ldots,i_q\}$ we have that $M_I \cdot \theta= M_{\{i_1\}} \cdots M_{\{i_1\}}\cdot \theta$, this implies that $M_I\cdot \theta \in \gfiber{\theta}$, proving the claim.
\end{proof}

\section{Three-Layer Bottleneck Fibers} 
\addcontentsline{apx}{section}{\protect\numberline{\thesection}Three-Layer Bottleneck Fibers} 
In a three-layer network, let $H_i$ denote the hyperplane of neuron $(1,i)$ and $B_j$ the bent hyperplane of neuron $(2,j)$, for $i \in [n_1]$ and $j \in [n_2]$. 
Moreover, for $P\in \complex_{\theta,\ell}$, let $S(P)=(S_1(P),\ldots,S_\ell(P))$, where $S_k(P) \subseteq [n_k]$ is the set of active neurons on $P$. 
In the remainder of this section, we let $\architecture=(d,n_1,n_2,m)$ with $n_1 \leq d$. 
In this case, for generic parameters, all activation patterns are realized on a non-empty polyhedron.
\begin{lemma}
\label{lem:first-layer-patterns-exist}
Let $\theta \in \gparamspace{\architecture}$ and assume $n_1\le d$. Then:  
\begin{enumerate}[leftmargin=*]
    \item For every subset $I\subseteq [n_1]$, there exists a unique region
    $P \in \complex_{\theta,1}^{d}$ such that
    $
    S_1(P)=I.
    $

    \item For every pair of disjoint subsets $I,Z \subseteq [n_1]$, there exists a non-empty face
    $\sigma \in \complex_{\theta,1}$ such that
    $
    \sigma \subseteq \bigcap_{k\in Z} H_k
$ and $S_1(\sigma) =I$ 

\end{enumerate}
\end{lemma}
\begin{proof} 
The statements follow from established results on hyperplane arrangements \citep[see, e.g.,][Chapter~2]{Orlik1992}. 
We offer a proof for completeness. 
Let
$
H_k=\{x\in\R^d \mid \preactivation{1,\theta}_k(x)=0\}
$
for $k\in[n_1]$. Since $\theta$ is generic, the first-layer hyperplanes
$H_1,\dots,H_{n_1}$ are in general position.

We first prove (2). Let $I,Z\subseteq[n_1]$ be disjoint and set
$
Y \coloneqq \bigcap_{k\in Z} H_k.
$
Since the arrangement is generic,
$
\dim Y=d-|Z|.
$
For every $j\in[n_1]\setminus Z$, the intersection $H_j\cap Y$ is an affine
hyperplane in $Y$, and these hyperplanes are again in general position. Their
number is
$
n_1-|Z| \le d-|Z|=\dim Y.
$
Hence the induced arrangement on $Y$ has all sign patterns. In particular, there
exists a unique region $\sigma$ of the induced arrangement on $Y$ such that
$
\preactivation{1,\theta}_k>0 \;\; \forall k\in I,
$ and $
\preactivation{1,\theta}_k<0 \;\; \forall k\in [n_1]\setminus (I\cup Z).
$
Then $\sigma$ is a non-empty face of $\complex_{\theta,1}$,
$
\sigma \subseteq \bigcap_{k\in Z} H_k,
$
and
$
S_1(\sigma)=I.
$
This proves (2).

Now we prove (1). Take $Z=\emptyset$ in (2). Then for every $I\subseteq[n_1]$
there exists a non-empty face $\sigma\in\complex_{\theta,1}$ with
$
S_1(\sigma)=I
$
and no vanishing first-layer preactivation. Hence $\sigma$ is full-dimensional,
so it is a region
$
P\in\complex_{\theta,1}^d.
$
This proves existence.

For uniqueness, a region of $\complex_{\theta,1}$ is uniquely determined by the
signs of the first-layer preactivations and no preactivation can vanish. If $S_1(P)=I$, then
$
\preactivation{1,\theta}_k>0 \;\; \forall k\in I,
$ and $
\preactivation{1,\theta}_k<0 \;\; \forall k\notin I.
$
Thus there is at most one such region. This proves (1).
\end{proof}

\subsection{Fibers of Non-Hiding Parameters}
\label{app:fibers_non_hiding}
We show that the first hidden layer is fixed across the fiber when all its hyperplanes are functionally visible. In a bottleneck architecture, linear independence of its weight vectors ensures that any parallelism among facets of the breakpoint complex can be attributed to a unique first-layer neuron.

\begin{restatable}{lemma}{lemidentifierregion}
\label{lem:identifierregion}
Let $\theta\in \gparamspace{\architecture}$ be generic. 
\begin{enumerate}[leftmargin=*]
    \item Let $P \in \complex_{\theta,1}^{d}$ and let $j\in[n_2]$. 
    If $P \cap B_{j}(\theta)\neq\emptyset$, then the facet
    $P \cap B_{j}(\theta)$ is parallel to a first-layer hyperplane 
    $H_i \in \hyperplanes{1}{\theta}$ if and only if 
    $S_1(P)=\{i\}$.
    
    \item Let $\sigma_1,\sigma_2 \in \complex_{\theta}^{d-1}$ be second-layer facets. 
    Then $\sigma_1$ and $\sigma_2$ are parallel if and only if there exist
    $i\in[n_1]$ and a region $P \in \complex_{\theta,1}^{d}$ such that
    $S_1(P)=\{i\}$
    and both $\sigma_1$ and $\sigma_2$ are contained in $P$.
\end{enumerate}
\end{restatable} 

\begin{proof}
On a fixed region $P \in \complex_{\theta,1}^{d}$, the first-layer activation pattern is constant. Hence each second-layer preactivation $\preactivation{2,\theta}_j$ restricts to an affine-linear function on $P$ with linear part
$
\sum_{k \in S_1(P)}\weights{2}_{jk}\,\weights{1}_k.
$
If $P\cap B_{2,j}(\theta)\neq\emptyset$, then the facet $P\cap B_{2,j}(\theta)$ is defined inside $P$ by the zero set of this affine-linear function, so its normal vector is given by the expression above.

For (1), the facet $P \cap B_{2,j}(\theta)$ is parallel to $H_i$ if and only if its normal vector is parallel to $\weights{1}_i$. Since $\theta$ is generic, the vectors $\weights{1}_1,\dots,\weights{1}_{n_1}$ are linearly independent and all coefficients $\weights{2}_{jk}$ are nonzero. Therefore
$
\sum_{k \in S_1(P)}\weights{2}_{jk}\,\weights{1}_k
$
is parallel to $\weights{1}_i$ if and only if exactly one term appears in the sum, namely the one with index $i$. This is equivalent to
$S_1(P) =\{i\}$.
This proves (1).

For (2), first suppose that there exist $i\in[n_1]$ and a region $P\in\complex_{\theta,1}^{d}$ with
$S_1(P) =\{i\}$
such that both $\sigma_1$ and $\sigma_2$ are contained in $P$. Then, by part (1), both facets are parallel to $H_i$, and hence parallel to each other.

Conversely, suppose that $\sigma_1$ and $\sigma_2$ are parallel. Since both arise from the second layer, there exist regions $P_1,P_2\in\complex_{\theta,1}^{d}$ and indices $j_1,j_2\in[n_2]$ such that
$
\sigma_r = P_r \cap B_{2,j_r}(\theta)$ for $ r=1,2.
$
By part (1), for each $r\in\{1,2\}$ there exists $i_r\in[n_1]$ such that

$S_1(P_r) =\{i_r\}$.
and $\sigma_r$ is parallel to $H_{i_r}$. Since $\sigma_1$ and $\sigma_2$ are parallel, they are parallel to the same first-layer hyperplane, so $i_1=i_2=:i$. Thus both $P_1$ and $P_2$ have the same sign pattern, namely the one with a single positive entry at index $i$. Since first-layer regions are uniquely determined by their sign patterns, it follows that $P_1=P_2$. Hence both facets are contained in the same region $P:=P_1=P_2$, which satisfies the claimed sign condition.
\end{proof}

We now show that the first hidden layer is generally invariant across the fiber, except for hiding parameters. For two weight matrices $\weights{\ell,\theta}$ and $\weights{\ell,\eta}$ we write $\weights{\ell,\theta} 
\sim \weights{\ell,\eta}$ if up to relabeling the rows of $\weights{\ell,\theta} $ are positive scalings of the rows of $
\weights{\ell,\eta}$.

\begin{restatable}{lemma}{lemfixfirstlayer}
\label{lem:fix_first_layer}
    Let $\theta = (\weights{\ell},\bias{\ell})_{\ell \in [3]} \in \gparamspace{\architecture}$ have no dead neurons and be not hiding. 
    If $\eta  \in \gfiber{\theta}\modulo{\sim}$, then 
    $\weights{1}\sim\weights{1,\eta}$ and $\bias{1} \sim \bias{1,\eta}$. 
\end{restatable} 

\begin{proof}
Since $\theta$ is not hiding, for every $i \in [n_1]$ there exists $j \in [n_2]$ and $x \in H_{1,i}$ such that $\preactivation{2,\theta}_j(x) >0$. 
Hence, by \Cref{thm:intersected_hyperplanes_stay}, it follows  that $\hyperplanes{1}{\theta}=\hyperplanes{1}{\eta}$. 
It remains to show that the orientation of the hyperplanes is also determined. 

    Assume that there is a neuron $i \in [n_2]$ such that $\bias{2}_i >0$. 
    Then, by genericity, $f_\theta$ is constant only on the region where all first-layer neurons are inactive, which determines the orientation of the hyperplanes. 
    Otherwise, there is a neuron $i \in [n_2]$ with $\bias{2}_i < 0$ and an index $j \in [n_1]$ such that $\weights{2}_{ij} >0$. 
    Let $P$ be the region with $S_1(P) =\{j\}$. 

    Since $\bias{2}_i < 0$ and $\weights{2}_{ij} >0$, the bent hyperplane $B_i$ intersects $P$. Let $\sigma$ be a facet of $P\cap B_i(\theta)$. Then $\sigma$ is parallel to $H_{1,j}$. Since $f_\eta=f_\theta$, the same facet $\sigma$ occurs in the breakpoint complex of $\eta$. By \Cref{lem:identifierregion}(1), the first-layer region of $\eta$ containing $\sigma$ must again be the region in which only neuron $j$ is active.  Therefore the orientation of all first-layer hyperplanes is fixed and this completes the proof. 
\end{proof}

This rigidity reduces the analysis of the three-layer fiber to the one-hidden-layer case in \Cref{sec:layerwise_fibers}. 

\begin{cor}
\label{cor:3layer_layerwise_reduction} 
Let $\theta = (\weights{\ell}, \bias{\ell})_{\ell \in [3]} \in \gparamspace{\architecture}$ satisfy the assumptions of \Cref{lem:fix_first_layer}. 
Then the generic fiber decomposes as 
\[
\gfiber{\theta} \modulo{\sim}  \ \ = (\{ (\weights{1}, \bias{1}) \} \times \gfiber{\theta'} )\modulo{\sim} , 
\]
where $\gfiber{\theta'}$ is the generic fiber of the one hidden layer network $g_{\theta'} \colon \orthant{n_1} \to \R^m$ given by $g_{\theta'}(y) = \weights{3}[\weights{2}y + \bias{2}]_+ + \bias{3}$, 
as characterized in  
\Cref{prop:generic_layerwise_fiber}. 
\end{cor}

\begin{proof}
By \Cref{lem:fix_first_layer}, 
the first-layer parameters are fixed across $\gfiber{\theta}$ 
up to trivial symmetry. 
The condition $f_\eta = f_\theta$ therefore requires the remaining layers to agree on the image of the first layer. 
Since $n_1 \le d$ and $\weights{1}$ has full rank, 
the map $x \mapsto [\weights{1}x + \bias{1}]_+$ is surjective onto the non-negative orthant $\orthant{n_1}$. 
Hence, the fiber is determined by functional equality on $\orthant{n_1}$, reducing the problem to the case analyzed in \Cref{sec:layerwise_fibers}. 
\end{proof}

\subsection{Fibers of Hiding Parameters} 
 \label{app:fibers_hiding}

We now characterize the fiber when a hyperplane is hidden. Even in this case, the weights $\weights{1}$ remain invariant if the second layer is sufficiently wide ($n_2 \ge 2$), as multiple second-layer neurons provide enough ``linear evidence'' of the hidden hyperplane’s direction.

\begin{restatable}{lemma}{lemtranslatinghiddenhyperplane}
\label{lem:translating_hidden_hyperplane}
Let $\theta = (\weights{\ell},\bias{\ell})_{\ell \in [3]} \in \gparamspace{\architecture}$ be generic and hiding hyperplane $j$. Let $T_j(\theta)$ be given as in \Cref{eq:translated_hyperplane}. 
Then 
\(
\{\eta \in \gfiber{\theta} \mid \weights{1,\eta} = \weights{1}\}
=
T_j(\theta)\cap \gparamspace{\architecture}. 
\)
\end{restatable} 

\begin{proof}
We first show
$
T_j(\theta)\cap \gparamspace{\architecture}
\subseteq
\{\eta \in \gfiber{\theta} \mid \weights{1,\eta}=\weights{1}\}.
$
Let $\eta \in T_j(\theta)\cap \gparamspace{\architecture}$. By \Cref{lem:translating_hidden_hyperplane_inclusion}, we have $\eta \in \fiber{\theta}$. Since moreover $\eta \in \gparamspace{\architecture}$, it follows that $\eta \in \gfiber{\theta}$. By construction, elements of $T_j(\theta)$ preserve the first-layer weights, so $\weights{1,\eta}=\weights{1}$.

For the converse inclusion, let 
$\eta \in \gfiber{\theta} \cap \{ \eta  \in \gparamspace{\architecture} \mid \weights{1,\eta} = \weights{1} \}$. 
Since the first-layer weights are fixed, the second-layer weights $\weights{2}$ are also fixed across $\gfiber{\theta}$. Indeed, 
on any region 
where all first-layer neurons are active (which exists by genericity and the assumption $n_1 \leq d$),  
the preactivations $\preactivation{\eta,2}_i$ are affine functions with linear part $\weights{2,\eta}_i\weights{1}$, which 
uniquely determines 
$\weights{2,\eta}_i$ since $\weights{1}$ has full rank. 
Now,
normalizing so that $\|\weights{2}_i\weights{1}\|=1$, 
\Cref{prop:tropweightofNNs} implies that the output weights $\weights{3}_{:i}$ are also uniquely determined. 
Hence $\weights{3}$ is also fixed across $\gfiber{\theta}$. 

By \Cref{thm:intersected_hyperplanes_stay}, all first-layer hyperplanes, except possibly the one corresponding to neuron $j$, are fixed across the fiber. 
Thus there exists $t \in \R$ such that 
\(
\bias{1,\eta} = \bias{1} + t e_j.
\)

 Since $\theta$ is hiding hyperplane $j$, for every $i \in [n_2]$ we have $\weights{2}_{i j} > 0$, $\weights{2}_{i k} < 0$ for $k \neq j$, and $\bias{2}_i < 0$. Let $P$ be the unique region of $\complex_{\theta,1}$ with
$S_1(P) =\{j\}$
Then, each second-layer bent hyperplane intersects $P$, 
 and by \Cref{lem:identifierregion}, the facets
\[
\sigma_i
=
P \cap B_i(\theta)
=
\{ x \in \R^d \mid \weights{2}_{i j}(\weights{1}_j x + \bias{1}_j) + \bias{2}_i = 0 \}
\]
are all parallel to $H_j$. 
By the same lemma, these facets $\sigma_i$ must be contained in some region $R \in \complex_{\eta,1}$ on which only neuron $j$ is active. 
Moreover, comparing with the corresponding expression for $\theta$ shows that necessarily
\(
\bias{2,\eta}_i = \bias{2}_i - t \weights{2}_{i j}
\)
for all $i \in [n_2]$.

Since $\weights{2}_{i j} > 0$ and $\bias{2}_i < 0$, the region $R$ exists only if 
\(
\bias{2,\eta}_i < 0
\)
for all $i$, i.e., 
\(
\bias{2}_i - t \weights{2}_{i j} < 0.
\)
This is equivalent to
\(
t > \max_{i \in [n_2]} \frac{\bias{2}_i}{\weights{2}_{i j}}
\). 
By the same argument as in the forward inclusion, this implies that $\activation{2,\eta}_i(x) = \activation{2,\theta}_i(x)$ for all $i$, hence $\eta \in T_j(\theta)\cap \gparamspace{\architecture}$.
\end{proof}
\begin{lemma}
\label{lem:hidinghides}
Let $\theta$ be generic and suppose that neuron $i\in[n_2]$ hides hyperplane $j\in[n_1]$. 
Then for every face $\sigma\in\complex_{\theta,1}$ we have that
$
B_i\cap \sigma\neq\emptyset
$ if and only if $
j\in S_1(\sigma).
$
\end{lemma}

\begin{proof}
Let $\sigma\in\complex_{\theta,1}$ be a face. On $\sigma$, the active set $S_1(\sigma)$ is constant, and the second-layer preactivation of neuron $i$ restricts to
$
\preactivation{2,\theta}_i(x)
=
\sum_{k\in S_1(\sigma)} \weights{2}_{ik}\preactivation{1,\theta}_k(x)+\bias{2}_i
$
for all $x\in \sigma$.

Assume first that $j\notin S_1(\sigma)$. Since neuron $i$ hides hyperplane $j$, we have
$
\weights{2}_{ik} \leq 0
$
for all $k\neq j$ and
$
\bias{2}_i<0
$
by genericity (otherwise the second-layer bent hyperplane corresponding to neuron $i$ would intersect the intersection of the first-layer hyperplanes violating supertransversality). Moreover,
$
\preactivation{1,\theta}_k(x)=\activation{1,\theta}_k(x)> 0
$
for all $k\in S_1(\sigma)$ and all $x\in \sigma$. Hence
$
\preactivation{2,\theta}_i(x)
=
\sum_{k\in S_1(\sigma)} \weights{2}_{ik}\activation{1,\theta}_k(x)+\bias{2}_i
<0
$
for all $x\in \sigma$. 
Therefore 
$\preactivation{2,\theta}_i$ has no zero on $\sigma$, and thus
$B_i\cap \sigma=\emptyset$. 

Now assume that $j\in S_1(\sigma)$. By \Cref{lem:first-layer-patterns-exist}, there exists a non-empty face
$
\tau \in \complex_{\theta,1}
$
such that
$
\tau \subseteq \sigma\cap\left( \bigcap_{k\in S_1(\sigma)\setminus\{j\}} H_k\right) 
$
and
$S_1(\tau)=\{j\}$. 
For all $x\in\tau$ we therefore have
$
\preactivation{2,\theta}_i(x)
=
\weights{2}_{ij}\preactivation{1,\theta}_j(x)+\bias{2}_i
$.
At $\tau\cap H_j$, this value equals $\bias{2}_i<0$.

On the other hand, since $\tau$ is unbounded in the direction where only neuron $j$ remains active, the function $\preactivation{1,\theta}_j$ is unbounded above on $\tau$. Hence there exists $x_+\in\tau$ such that
$
\weights{2}_{ij}\preactivation{1,\theta}_j(x_+) > -\bias{2}_i$, 
and therefore
$
\preactivation{2,\theta}_i(x_+)>0$. 
By continuity, $\preactivation{2,\theta}_i$ must vanish at some point of $\tau$, and since $\tau\subseteq \sigma$, it follows that
$
B_i\cap \sigma\neq\emptyset
$, proving the claim. 
\end{proof}

\fiberfromhiding*

\begin{proof}[Proof of \Cref{prop:fiber_from_hiding}] 
 By \Cref{lem:translating_hidden_hyperplane}, it suffices to show that $\eta \in \gfiber{\theta} \modulo{\sim}$ implies $\weights{1,\eta} \sim \weights{1}$. 
 By \Cref{thm:intersected_hyperplanes_stay} and \Cref{lem:hidinghides}, all first-layer hyperplanes of $\eta$ coincide with those of $\theta$ except possibly for 
the $j$-th hyperplane. Thus 
\(
\hyperplanes{1}{\eta}
=
(\hyperplanes{1}{\theta} \setminus \{H_j\}) \cup \{\widehat H_j\}
\)
for some hyperplane $\widehat H_j$. 

By \Cref{lem:first-layer-patterns-exist}, let $P \in \complex_{\theta,1}$ be the region where only neuron $j$ is active in the first layer. 
It follows from Lemma \ref{lem:identifierregion} that
for any second-layer neuron $i\in[n_2]$
the facet  
\[
\sigma_i
=
P \cap B_i(\theta)
=
\{ x \in P \mid \weights{2}_{i j}(\weights{1}_j x + \bias{1}_j) + \bias{2}_i = 0 \}
\]
is  parallel to $H_j=\{x\mid \weights{1}_j x + \bias{1}_j = 0\}$. 
Since $n_2 \ge 2$ and by \Cref{lem:hidinghides} all bent hyperplanes intersect $P$, there are at least two such parallel facets inside $P$. 
As these facets must be facets of the breakpoint complex for each parameter realizing the same function,
it follows again by \Cref{lem:identifierregion} that $\widehat H_j$ is parallel to $H_j$. 
Moreover, by the same lemma, only neuron $j$ can be active on the region containing the two parallel facets, which implies that the orientation of the hyperplanes is also determined.  
Thus, the first-layer weight matrix $\weights{1}$ is fixed (up to trivial symmetries) across $\gfiber{\theta}$. 
The claim now follows from \Cref{lem:translating_hidden_hyperplane}. 
\end{proof}

\fiberfromhidingsingleneuron*

\begin{proof}[Proof of \Cref{prop:hiding_fiber_single_second_layer}]
We first show
$
\left(
\bigcup_{I \subseteq [n_1]\setminus \{j\}}
\left(T_j(M_I \cdot \theta)\cap \gparamspace{\architecture}\right)
\right)\modulo{\sim} \ \
\subseteq \ 
\gfiber{\theta}\modulo{\sim}
$. 

Let $\eta$ belong to the left-hand side. Then there exist
$I \subseteq [n_1]\setminus\{j\}$ and
$
\widetilde\eta \in T_j(M_I\cdot \theta)\cap \gparamspace{\architecture}
$
such that $\eta\sim\widetilde\eta$. By \Cref{prop:flipping_hidden_hyperplane_inclusion}, we have
$
M_I\cdot\theta \in \fiber{\theta},
$
and by \Cref{lem:translating_hidden_hyperplane_inclusion}, we have 
$
 T_j(M_I\cdot \theta)\in \fiber{M_I\cdot\theta}.
$
Hence $\widetilde\eta\in\fiber{\theta}$. 
Moreover, since $\widetilde\eta\in\gparamspace{\architecture}$, it follows that $\widetilde\eta\in\gfiber{\theta}$. 
Therefore $\eta\in\gfiber{\theta}\modulo{\sim}$.

For the \textbf{reverse inclusion}, let $\eta\in\gfiber{\theta}\modulo{\sim}$. 
 We will show that 
$
\eta \in T_j(M_I\cdot\theta)
$
for some $I\subseteq[n_1]\setminus\{j\}$.

\medskip
\textbf{$\eta$ also hides hyperplane $j$:}
Since $\theta$ hides hyperplane $j$ and $f_\eta=f_\theta$, \Cref{thm:intersected_hyperplanes_stay} and \Cref{lem:hidinghides} imply that
$
H_k(\eta)=H_k(\theta)
$
for all $k\neq j$. Thus, only the $j$-th hyperplane (the hidden hyperplane) may differ. 
Because $f_\eta=f_\theta$, the breakpoint sets agree:
$
\breakpoints{f_\eta}=\breakpoints{f_\theta}
$. 
Since $H_j(\theta)$ is hidden in $\theta$, it does not occur in $\breakpoints{f_\theta}$. 
Since $\eta$ is generic, the corresponding first-layer hyperplane of $\eta$ is also hidden, which implies 
$
\weights{2,\eta}_{1j}>0,
$
$
\weights{2,\eta}_{1k}<0
$
for all $k\neq j$,
and
$
\bias{2,\eta}_1<0$, by \Cref{lem:hidinghides}. Replacing $\eta$ by an equivalent representative if necessary, we may assume that $\eta$ is normalized, meaning that the second layer weights are either $1$ or $-1$. 
Hence, 
$
\weights{2,\eta}_{1j}=1
$
and
$
\weights{2,\eta}_{1k}=-1
$
for all $k\neq j$. 

\medskip
\textbf{Possible choices for $\weights{1,\eta}_j$:}
Because $f_\eta=f_\theta$, the visible second-layer bent hyperplane agrees:
$
B_1(\eta)=B_1(\theta)
$. 
For any first-layer region $Q\in\complex_{\theta,1}$ with active set $S\subseteq[n_1]$, the linear part of $\preactivation{2,\theta}_1$ on $Q$ is
$
\sum_{k\in S}\weights{2,\theta}_{1k}\weights{1,\theta}_k
$. 
Since $\theta$ is normalized and hides hyperplane $j$, we have
$
\weights{2,\theta}_{1j}=1
$
and
$
\weights{2,\theta}_{1k}=-1
$
for $k\neq j$.
Hence every visible facet of $B_1(\theta)$ has normal vector of the form
$
\weights{1,\theta}_j-\sum_{k\in I}\weights{1,\theta}_k
$, for some $I\subseteq[n_1]\setminus\{j\}$. 

By \Cref{lem:first-layer-patterns-exist}, there exists a unique region $R\in\complex_{\eta,1}^d$ on which only neuron $j$ is active under $\eta$. 
Since $\eta$ is normalized, the linear part of $\preactivation{2,\eta}_1$ on $R$ is
$
\weights{2,\eta}_{1j}\weights{1,\eta}_j=\weights{1,\eta}_j
$. 
Let $\sigma=R\cap B_1(\eta)$, which is not empty by \Cref{lem:hidinghides} because $\eta$ hides hyperplane $j$. 
Because $B_1(\eta)=B_1(\theta)$, the same facet $\sigma$ is a visible facet of $B_1(\theta)$. Hence there exists $I\subseteq[n_1]\setminus\{j\}$ such that
$
\weights{1,\eta}_j
=
\weights{1,\theta}_j-\sum_{k\in I}\weights{1,\theta}_k
$. 

\medskip
\textbf{Compensate the choice of weights:}
Let $Q_{\mathrm{all}}\in\complex_{\theta,1}$ be the unique region on which all first-layer neurons are active under $\theta$, which again exists by \Cref{lem:first-layer-patterns-exist}. On this region, the linear part of $\preactivation{2,\theta}_1$ is
$
\weights{1,\theta}_j-\sum_{k\in[n_1]\setminus\{j\}}\weights{1,\theta}_k.
$
Let $\sigma'= Q_{\mathrm{all}}\cap B_1(\theta)$. Since $f_\eta=f_\theta$, the same facet $\sigma'$ occurs in the breakpoint complex of $\eta$. 
Let $Q_\eta\in\complex_{\eta,1}$ be the first-layer region containing $\sigma'$, and let $S_1(Q_\eta)\subseteq[n_1]$ be the set of active neurons on $Q_\eta$.
For each $k\neq j$, the hyperplanes $H_k(\eta)=H_k(\theta)$ coincide, so there exists $\varepsilon_k\in\{\pm1\}$ such that
$
\weights{1,\eta}_k=\varepsilon_k\,\weights{1,\theta}_k
$
for all $k\neq j$.

Since $\eta$ is normalized, the linear part of $\preactivation{2,\eta}_1$ on $Q_\eta$ is
\begin{equation}\label{exp:linear_part}
    \mathbf{1}_{\{j\in S_1(Q_\eta)\}}\weights{1,\eta}_j
-
\sum_{k\in S_1(Q_\eta)\setminus\{j\}}\weights{1,\eta}_k,
\end{equation}
where $$\mathbf{1}_{\{j \in S_1(Q_\eta)\}} = \begin{cases}
    1 &  j \in S_1(Q_\eta), \\ 
    0 & j \notin S_1(Q_\eta).
\end{cases}$$
Substituting
$
\weights{1,\eta}_j
=
\weights{1,\theta}_j-\sum_{k\in I}\weights{1,\theta}_k
$
and
$
\weights{1,\eta}_k=\varepsilon_k\,\weights{1,\theta}_k
$
for $k\neq j$, we obtain that Expression \ref{exp:linear_part} is equal to

$$
\mathbf{1}_{\{j\in S_1(Q_\eta)\}}
\left(
\weights{1,\theta}_j-\sum_{k\in I}\weights{1,\theta}_k
\right)
-
\sum_{k\in S_1(Q_\eta) \setminus\{j\}}\varepsilon_k\,\weights{1,\theta}_k.
$$
Since $\sigma'$ is also a facet of $Q_{\mathrm{all}}\cap B_1(\theta)$, this must also equal
$
\weights{1,\theta}_j-\sum_{k\in[n_1]\setminus\{j\}}\weights{1,\theta}_k.
$
Because the vectors
$
\weights{1,\theta}_1,\dots,\weights{1,\theta}_{n_1}
$
are linearly independent, it follows that
$
j\in S_1(Q_\eta)
$
and
$S_1(Q_\eta)=[n_1]\setminus I$. 
Indeed, if $k\in I$, then the coefficient of $\weights{1,\theta}_k$ is already $-1$ from the first sum, so $k\notin S_1(Q_\eta)$. If $k\notin I$, then $k$ must belong to $S_1(Q_\eta)$ in order to contribute the coefficient $-1$ on the right-hand side.

Now $Q_{\mathrm{all}}$ is the unique region in which every fixed neuron $k\neq j$ is active under $\theta$. Since the hyperplanes $H_k(\eta)=H_k(\theta)$ are fixed  and $S_1(Q_\eta)=[n_1]\setminus I$, it follows that the orientation of neuron $k$ is preserved for $k\notin I$ and reversed for $k\in I$. Therefore
$\weights{1,\eta}=M_I\weights{1,\theta}$. 

Since both $\eta$ and $M_I\cdot\theta$ are normalized and hide hyperplane $j$, their second-layer weights coincide. Since the realized functions agree, by \Cref{prop:tropweightofNNs}, the output weights coincide as well. Therefore \Cref{lem:translating_hidden_hyperplane} applied to $M_I\cdot\theta$ implies that $\eta$ differs from $M_I\cdot\theta$ only by a translation of the hidden hyperplane. Hence
$\eta \in T_j(M_I\cdot\theta) \cap \gparamspace{\architecture}$, 
and therefore
$
\eta
\in
\left(
\bigcup_{I \subseteq [n_1]\setminus\{j\}}
\left(T_j(M_I\cdot\theta)\cap \gparamspace{\architecture}\right)
\right)\modulo{\sim}
$. 
This proves the reverse inclusion.
\end{proof}

\subsection{Characterization of Identifiability} 

\polytimeequivalencebottleneck*

\begin{proof}[Proof of \Cref{prop:polytime_equivalence_bottleneck}]
We describe a decision procedure and verify that each step can be carried out in polynomial time.

First inspect the sign pattern of the rows of $[\weights{2},\bias{2}]$ in order to determine whether $\theta$ and $\eta$ are hiding or non-hiding. Moreover, dead neurons can be removed. By the bottleneck and genericity assumption, the image of the first hidden layer is always the entire nonnegative orthant and hence a neuron is dead if and only if all its incoming weights and its bias is nonpositive. 

\medskip
\noindent
\textbf{Case 1: non-hiding.}
By \Cref{lem:fix_first_layer}, if $f_\theta=f_\eta$, then the first-layer parameters of $\theta$ and $\eta$ must agree up to permutation and positive rescaling of rows. Thus we first compare the first-layer hyperplanes of $\theta$ and $\eta$, equivalently, we check whether the rows of $[\weights{1},\bias{1}]$ match up to permutation and positive scaling. If not, then $f_\theta\neq f_\eta$.

Assume now that the first layers agree in this sense. By \Cref{cor:3layer_layerwise_reduction}, equality of $f_\theta$ and $f_\eta$ is equivalent to equality of the induced one-hidden-layer networks on $\orthant{n_1}$. By \Cref{lem:hyperplane_representation}, this is equivalent to equality of the corresponding weighted hyperplane arrangements together with equality of the affine map on one base region. These quantities are explicitly computable from the parameters using sign comparisons and linear-algebraic operations, and can therefore be checked in polynomial time.

\medskip
\noindent
\textbf{Case 2: hiding with $n_2\ge 2$.}
By \Cref{prop:fiber_from_hiding}, if $f_\theta=f_\eta$, then both parameters hide the same first-layer hyperplane $H_j$, all remaining first-layer hyperplanes agree up to permutation and positive rescaling, and $\eta$ differs from $\theta$ only by a translation of the hidden hyperplane. Thus we first identify the hidden index $j$ from the sign pattern of $[\weights{2},\bias{2}]$ and check that the same index occurs for $\eta$. We then compare all non-hidden first-layer hyperplanes. If these checks fail, then $f_\theta\neq f_\eta$.

If they succeed, then by \Cref{lem:translating_hidden_hyperplane} equivalence is equivalent to the existence of a scalar $t$ such that the first- and second-layer biases differ exactly as in \Cref{eq:translated_hyperplane}. This is a system of linear equalities and inequalities in one unknown $t$, and can therefore be checked in polynomial time.

\medskip
\noindent
\textbf{Case 3: hiding with $n_2=1$.}
Again we first identify the hidden index $j$ and compare all non-hidden first-layer hyperplanes. Their relative orientations determine the corresponding sign flips. By \Cref{prop:hiding_fiber_single_second_layer}, once the non-hidden hyperplanes are matched, the first-layer parameter of the hidden neuron is uniquely determined by the visible second-layer bent hyperplane. Equivalently, the subset $I\subseteq[n_1]\setminus\{j\}$ in the transformation $M_I$ is recovered from the orientations of the matched non-hidden hyperplanes. Thus no search over subsets is required.

Having recovered $I$, equivalence is equivalent to checking whether $\eta$ differs from $M_I\cdot\theta$ by a translation of the hidden hyperplane as in \Cref{eq:translated_hyperplane}. As in Case 2, this reduces to checking a system of linear equalities and inequalities in one scalar $t$, and is therefore polynomial-time decidable.

\medskip
These three cases exhaust all generic parameters in the bottleneck setting. Hence the procedure decides whether $f_\theta=f_\eta$ in polynomial time.
\end{proof}
\corvolumebound*

\begin{proof}[Proof of \Cref{cor:volume-bound}]
By \Cref{thm:charaterization_identifiablity_bottleneck}, non-identifiability of a generic parameter is completely determined by the sign pattern of the rows of $[\weights{2},\bias{2}]$. Since generic parameters form a full-measure subset of $\paramspace{\architecture}$ and all orthants have the same relative volume, 
it suffices to compute probabilities under the uniform distribution on sign patterns.

\medskip

A necessary condition for identifiability is that no row of $[\weights{2},\bias{2}]$ only has negative entries. For a fixed row, this happens with probability $2^{-(n_1+1)}$, and for all rows these events are independent. Hence,
\[
\mathbb{P}(\text{no row is all negative})
= (1 - 2^{-(n_1+1)})^{n_2}.
\]
Since identifiability implies this condition, we obtain
\[
\frac{\operatorname{vol}(\mathcal{I})}{\operatorname{vol}(\paramspace{\architecture})}
\le (1 - 2^{-(n_1+1)})^{n_2}.
\]

\medskip

By \Cref{thm:charaterization_identifiablity_bottleneck}, a parameter is non-identifiable if at least one of the statements 1,2 or 3 holds. We bound the probability of each event.

\smallskip

\noindent\emph{Event 1.}
For each row, the probability that all entries are negative is $2^{-(n_1+1)}$. By the union bound over the $n_2$ rows,
\[
\mathbb{P}(\text{Event 1}) \le n_2 \, 2^{-(n_1+1)}.
\]

\smallskip

\noindent\emph{Event 2.}
For each row, the probability that all entries are positive is again $2^{-(n_1+1)}$, implying
\[
\mathbb{P}(\text{Event 2}) \le n_2 \, 2^{-(n_1+1)}.
\]

\smallskip

\noindent\emph{Event 3.}
Fix $j \in [n_1]$. For a single row, the probability of being positive in column $j$ and negative elsewhere is $2^{-(n_1+1)}$.
By independence across rows, the probability that all $n_2$ rows satisfy this pattern is $2^{-n_1 n_2}$. Taking a union bound over $j \in [n_1]$ yields
\[
\mathbb{P}(\text{Event 3}) \le n_1 \, 2^{-n_1 n_2}. 
\]

\medskip

By the union bound, 
\[
\mathbb{P}(\text{non-identifiable})
\le 
\sum_{i=1}^3\mathbb{P}(\text{Event i}) , 
\]
and therefore
\[
\frac{\operatorname{vol}(\mathcal{I})}{\operatorname{vol}(\paramspace{\architecture})}
= 1 - \mathbb{P}(\text{non-identifiable})
\ge
1 - 2n_2 \, 2^{-(n_1+1)}  - n_1 \, 2^{-n_2n_1} = 1 - n_2 \, 2^{-n_1}  - n_1 \, 2^{-n_2n_1}.
\]

\medskip

Combining the upper and lower bounds proves the claim.
\end{proof}

\section{Implications for Deeper Networks}
\addcontentsline{apx}{section}{\protect\numberline{\thesection}Implications for Deeper Networks} 

\subsection{Fibers Are Not Localizable}

\proplocalizable* 

\begin{proof}[Proof of \Cref{prop:localizable}] 
The idea is to construct parameters that satisfy the sign condition for identifiability from \Cref{thm:charaterization_identifiablity_bottleneck}, while ensuring that the image of $f_2 \circ f_1$ is bounded so that the neurons of $f_3$ act linearly on this image. 

Let $\weights{1} = I_d$, $b^{(1)} = 0$, $\weights{2}_{ij} <0$, $\bias{2}\geq 0$, and $\weights{3}_{ij} <0$, $\bias{3} = t v$ for some positive vector $v$ and positive scalar $t$. 
All these layers and also the fourth layer are chosen generically in the sense of \Cref{def:generic_parameter}. 

By construction, the rows of $[\weights{2},b^{(2)}]$ and $[W^{(3)},b^{(3)}]$ have mixed signs, no row is strictly negative nor strictly positive, and no column satisfies the hiding condition of \Cref{thm:charaterization_identifiablity_bottleneck}. 
Hence $f_2 \circ f_1$ and $f_3 \circ f_2$ are identifiable. However, the image of $f_2 \circ f_1$ is bounded.

Hence, choosing $t>0$ sufficiently large ensures that $W^{(3)}y+b^{(3)}$ is strictly positive on $(f_2 \circ f_1)(\R^d)$, so the ReLU in the third layer is inactive. Hence, it acts like a linear map. Thus, for $X \in \GL^+_{n_3}$ with small enough norm, we have that replacing the weights with $X\weights{3},X\bias{3}$ and $\weights{4}X^{-1}$ does not change the function. Hence $f_\theta$ is not identifiable. 
\end{proof}

\subsection{Conserved Quantities from GL-Action on Linear Neurons}
\label{subsec:conserved_quantities}

We now show that the symmetries presented in 
\Cref{sec:layerwise_fibers} 
which are induced by the monoid action of $\GL^+$, give rise to locally conserved quantities. 
We first make a statement that holds for linear neural networks, and then show that this can be applied to the setting of \Cref{prop:empty_or_dim}. 

\begin{prop}
\label{prop:conserved_q_linear}
    Let $\theta=\{(U,V,u,v) \in \R^{m \times J} \times \R^{J \times d} \times \R^m \times \R^J \}\in\Theta$ denote the parameters of a linear network $f:\R^d\rightarrow\R^m$ given by $f(x)=U(Vx+v)+u.$ 
    For fixed $A\in\R^{m\times d}$ and $ b\in\R^m,$ let $D_{A,b}\subseteq\Theta$ be the subspace given by 
    \begin{equation}\label{eq:cond_grad_flow}
D_{A,b}=\left\{(U,V,u,v) \in \R^{m \times J} \times \R^{J \times d} \times \R^m \times \R^J \ \vert \   UV = A, \quad Uv + u = b\right\}.
\end{equation}
Then  \begin{equation}
    U^\top U-VV^\top-vv^\top
\end{equation} is (locally) preserved on $D_{A,b}$ through gradient flow $\dot{\theta}$.
\end{prop}

\begin{proof} 
  Let $X:\R\rightarrow\GL_J$ be a curve into the space of invertible matrices. Recall that $\GL_J$ acts on $D_{A,b}$ via the group action $X\cdot(U,V,u,v)=(UX^{-1},XV,u,Xv)$. 
  Then $\gamma(s)=X(s)\cdot\theta$ is a curve into the fiber of a parameter $\theta$. 
  Since $\langle\dot{\gamma}(s),\dot{\theta}\rangle=0,$ a function $C:\Theta\rightarrow\R$ with $\nabla_\theta C=\dot{\gamma}(s)$ must be conserved. 
  For simplicity, we will write $Y(s)=X(s)^{-1}$. 
From \begin{equation}
    \gamma(s)=(UY(s),X(s)V,u,X(s)v)
\end{equation}
it follows that
\begin{equation}
    \dot{\gamma}(s)=(U\dot{Y}(s),\dot{X}(s)V,0,\dot{X}(s)v).
    \label{eq:velocity}
\end{equation}
Equation \ref{eq:velocity} implies that 
\begin{eqnarray*}
    \nabla_{V_{ij}}C&=&\sum_t\dot{X}_{it}V_{tj},\\
\nabla_{U_{ij}}C&=&\sum_tU_{it}\dot{Y}_{tj},\\
\nabla_{v_i}C&=&\sum_t\dot{X}_t v_t.
\end{eqnarray*}
Integrating this, we obtain 
\begin{equation}
\label{eq_C_conserved}
C(\theta)=\frac{1}{2}\cdot\sum_{i,j,t}\left(V_{ji}V_{ti}\dot{X}_{tj}+U_{it}U_{ij}\dot{Y}_{tj}\right)+\frac{1}{2}\cdot\sum_{i,t}\dot{X}_{it}v_tv_i.
\end{equation}
Let $E_{ij}$ be the $J\times J$ matrix with a $1$ in its $(i,j)$-th entry and zeros elsewhere. 
Consider now 
$$
X(s)=I_J+s(E_{mk}+E_{km}), \ \ s\geq0.
$$ 
This $X(s)$ is invertible for $s\neq\pm1$, with $\dot{Y}(0)=-(E_{mk}+E_{km})$ and $\dot{X}(0)=E_{mk}+E_{km}$. 
Plugging these values into \eqref{eq_C_conserved}, we obtain 
$C(\theta)=(VV^\top)_{km}-(U^\top U)_{km}+(vv^\top)_{km}$. 
Running through all indices $k,m$ this implies that $VV^\top-U^\top U+vv^\top$.  
\end{proof}

\begin{prop}
    Let $f_\theta = f^{(1)}_{\theta_1} \circ (f^{(2)}_{\theta_2} + f^{(3)}_{\theta_3}) \circ f^{(4)}_{\theta_4}$ be a neural network (or even just a parameterized function) and assume that on $S \subseteq \Theta_2$ we have a conserved quantity $C_2(\theta_2)$. Then $C((\theta_i)_{i=1,\ldots,4}) = C_2(\theta_2)$ is a conserved quantity on $\Theta_1 \times S \times \Theta_3 \times \Theta _4$.
\end{prop}

\begin{proof}
    Since $C_2(\theta_2)$ is conserved on $S\subseteq\Theta_2,$ we have that whenever $\theta_2\in S$ 
    \begin{equation*}
        \langle\nabla_{\theta_2}C_2(\theta_2),\dot{\theta}_2\rangle=0, 
    \end{equation*}
    where $\dot{\theta}_2=\pi_2(\dot{\theta})$ is the projection of $\dot{\theta}$ on its second component. 
    Since $\nabla_\theta C(\theta)=(0,\nabla_{\theta_2}C_2(\theta_2),0,0)$, it immediately follows that 
    \begin{equation*}
        \langle\nabla_{\theta}C(\theta),\dot{\theta}\rangle=\langle\nabla_{\theta_2}C_2(\theta_2),\dot{\theta}_2\rangle=0, 
    \end{equation*} 
    whenever $\theta_2\in S$ or equivalently $\theta\in\Theta_1 \times S \times \Theta_3 \times \Theta _4$. 
    This implies that $C(\theta)$ is conserved under $\dot{\theta},$ as long as $\theta\in\Theta_1 \times S \times \Theta_3 \times \Theta _4$. 
\end{proof}
Noting that the curve used in the proof of Proposition~\ref{prop:conserved_q_linear} is contained in $\GL^+$, we obtain the locally conserved quantity given in Theorem \ref{th:locally_conserved}.

\subsection{Conserved Quantities from Layer Concatenation}
\label{app:conserved_layer_concat}

\begin{prop}
The continuous symmetry described in \Cref{sec:fibers_from_concatenation} does not induce a conserved quantity of gradient flow. 
\end{prop}

\begin{proof} 
Consider $T_j(\theta)$ as in \Cref{lem:translating_hidden_hyperplane_inclusion}. 
This immediately yields the curve 
$$
\gamma(t):(\max_{i \in [n_2]} \frac{\bias{2}_i}{\weights{2}_{i j}}, \infty)\rightarrow T_j(\theta)
$$ 
given by 
$$\gamma(t)=(\weights{1},
\bias{1}+t e_j,
\weights{2},
\bias{2}-t(\weights{2}_{1j},\ldots,\weights{2}_{n_2 j})^\top,
\weights{3},\bias{3}). 
$$
Since the gradient flow $\dot{\theta}=-\nabla_\theta\mathcal{L}$ is orthogonal to the fiber $\gfiber{\theta},$ we have $\langle\dot{\gamma}(t),\dot{\theta}\rangle=0$. 
We want to use this to find a conserved quantity $C:\Theta\rightarrow\R$ with 
\begin{equation}
\label{eq:conserved_quantity_orth}
\nabla_\theta C=\dot{\gamma}(0). 
\end{equation}
Evaluating $\dot{\gamma}(s)$, we require $\nabla_{\weights{2}_{i,j}}C=0$ for all $i$. 
However, we have that $\frac{\partial C}{\partial\bias{2}_{i}}=-\weights{2}_{ij}$. 
Observing that $C$ must have continuous partial derivatives, implying that the mixed partial derivatives must agree, we conclude that there exists no $C$ satisfying \eqref{eq:conserved_quantity_orth}. 
Due to the condition $\weights{2}_{ij}>0,\weights{2}_{ik}\leq0$ for all $k\neq j$ per layer there can exist at most one hidden hyperplane and thus at most one such one-dimensional fiber $T_j(\theta)$. 
Hence, this symmetry induces no conserved quantities. 
\end{proof}

\section{Discussion: Beyond the Generic Bottleneck Regime}
\addcontentsline{apx}{section}{\protect\numberline{\thesection}Discussion: Beyond the Generic Bottleneck Regime} 
\label{app:discussion_beyond_bottleneck} 

Our main results provide a complete description of generic fibers for three-layer bottleneck architectures. 
In this appendix, we discuss how the underlying ideas may extend beyond this setting and why explicit descriptions can be expected to become substantially more complicated. 

A useful guiding principle is that the symmetry mechanisms identified in this work persist more generally. 
In particular, both layerwise symmetries and hiding symmetries remain meaningful in wider and deeper architectures. 
The challenge lies not in their existence, but describing their interactions with other symmetries and in assessing whether they form a complete description. 

\subsection{Wider Architectures}

The bottleneck assumption plays a special geometric role, as it implies that for generic parameters the image of the first hidden layer is the full nonnegative orthant. 
This reduces the analysis of the second layer to that of 
a single weighted hyperplane arrangement on a fixed domain. 

Once the bottleneck assumption is dropped, this picture changes fundamentally. 
The image of the first hidden layer is no longer the full orthant, but a union of polyhedra in activation space, with up to $O(n^d)$ such cells in general. 
To determine which geometric features of the first layer are preserved by the next layer and which are hidden, one must analyze how the next-layer hyperplanes intersect this polyhedral image. 

In principle, the same general philosophy should still apply locally on each image cell: 
one can ask whether a hyperplane of the next layer 
intersects the relevant part of the image, 
whether it is visible in the realized function, 
and whether it constrains the geometry of the previous layers. 
However, the compact hyperplane representation available in the bottleneck case no longer applies. 
Instead, one must track many activation patterns separately, and any resulting semi-algebraic description of the fiber is therefore expected to grow combinatorially. 

Moreover, wider architectures may introduce additional continuous degrees of freedom that do not arise in the bottleneck case. 
If the intersection of the image of one layer with a hyperplane of the next fails to span the entire image cell, 
then the hyperplane may admit deformations that do not change its effect on the image of the previous layers. 
For instance, one can consider rotations that leave its intersection with the image invariant. 
In such a regime, the relevant symmetries are no longer directly determined by the weight sign pattern alone.

\subsection{Deeper Architectures}

For deeper networks, the main difficulty is not only combinatorial. 
Our non-localizability result (\Cref{prop:localizable}) shows that parameter equivalence cannot, in general, be reduced to pairwise layer compositions. 
This indicates that, in deeper architectures, one must track the image of an entire network prefix and study how its image is filtered through all subsequent layers.

In other words, explicit fiber descriptions are not expected to be layerwise decomposable. 
The admissible transformations at a given layer depend on the geometry induced by all preceding layers and on how this geometry is preserved, hidden, or altered by all subsequent layers. 

A further difficulty is that breakpoint features may no longer be uniquely attributable to a single layer. 
In the three-layer bottleneck case, visibility arguments allow us to identify which breakpoint facets come from the first layer and which from the second. 
In deeper architectures, this full separation may fail. As a result, new discrete symmetries may emerge from different ways of assigning the same realized breakpoint structure to different layers. 
Such discrete symmetries may also interact nontrivially with continuous symmetries, for example those arising from linear neurons.

\subsection{Nongeneric Parameters}

The nongeneric case is already substantially more intricate, even for three-layer architectures. 
The rigidity results used throughout our analysis rely on genericity in two ways: 
they ensure that hyperplanes intersect in the expected codimensions and prevent algebraic degeneracies in the masked weight products. 

Once genericity is dropped, both mechanisms can fail. Hyperplanes may coincide or meet in unexpected dimensions, so that different neurons may contribute indistinguishably to the same geometric feature. 
In particular, neurons can no longer be assigned unambiguously to visible breakpoints. 
This opens the door to many further degeneracies, including additional discrete symmetries arising from coincident, canceling, or partially hidden bent hyperplanes. 

Thus, while the generic bottleneck case admits an explicit description in terms of visible hyperplanes, orientations, and affine compensation, the nongeneric case is expected to require a substantially richer case distinction.

\subsection{Complexity Perspective}

The geometric complications described above are supported by complexity-theoretic considerations. 
In particular, our polynomial-time equivalence test appears to rely essentially on both the bottleneck and genericity assumptions.

First, if the bottleneck assumption is dropped, then deciding functional equivalence of three-layer ReLU networks is already coNP-hard. 
This follows from hardness results for positivity of one-hidden-layer ReLU networks. 
For example, \citet{froese2026parameterized} show that, given an instance of \textsc{Multicolored Clique}, one can construct in polynomial time a two-layer ReLU network computing a function 
$
f\colon \R^k\to\R
$
such that 
$
\max_{x\in\R^k} f(x)=k+\binom{k}{2} 
$
in the yes-case, whereas 
$
\max_{x\in\R^k} f(x)\le k+\binom{k}{2}-1
$
in the no-case. 
By adding one further ReLU neuron that thresholds the output at the midpoint of this gap, one obtains a three-layer ReLU network whose realized function is identically zero in the no-case and nonzero in the yes-case. 
Deciding whether this network computes the zero function is therefore coNP-hard, and equivalently so is deciding functional equivalence with the zero network. 
Since the first hidden layer in this construction can be substantially wider than the input dimension, the resulting hard instances lie outside the bottleneck regime. 

Second, hardness persists even within bottleneck architectures if genericity is dropped. 
The reduction of \citet{pmlr-v291-froese25a} for positivity of one-hidden-layer ReLU networks constructs instances in input dimension $d$ with only $O(d^2)$ hidden neurons. 
Such an instance can be embedded into a larger ambient space of dimension $O(d^2)$ by composing with a projection onto the first $d$ coordinates. 
This preserves positivity, while the resulting network has first hidden layer width at most the ambient dimension, and is therefore a bottleneck architecture and necessarily nongeneric. 
In this way, one obtains coNP-hardness of functional equivalence even within the class of bottleneck architectures once genericity is dropped.

Finally, it seems unlikely that restricting to generic parameters alone restores tractability in the non-bottleneck regime. 
The hardness construction of \cite{froese2026parameterized} is separated by a constant gap, and all potential maximizers lie in a bounded region. 
This suggests that, after suitably stabilizing the construction outside that region, for example by adding neurons that force the network to attain negative values there, 
sufficiently small perturbations of the realized function should preserve the distinction between yes- and no-instances. 
Since generic parameters are dense, this provides strong evidence that coNP-hardness should persist even under a genericity restriction. 
However, turning this intuition into a formal reduction would require an efficient deterministic procedure for perturbing arbitrary instances into generic ones. 
At present, we do not know how to construct such perturbations effectively, since genericity requires avoiding finitely many, but exponentially many, algebraic varieties. 

From this perspective, the loss of compact geometric descriptions outside the generic bottleneck setting is not merely an artifact of our proof technique. 
Rather, it is consistent with the fact that deciding functional equivalence itself becomes hard in these broader regimes. 
In particular, one should not expect similarly explicit and efficiently checkable semi-algebraic descriptions of fibers once these assumptions are removed. 
The combinatorial and geometric blow-up described above may therefore reflect a genuine increase in problem complexity. 
In this sense, the generic bottleneck setting isolated in this paper appears to lie near a tractability boundary for explicit and efficient descriptions of ReLU network fibers.


\end{document}